\documentclass[twoside]{article}

\usepackage[accepted]{aistats2026}
\usepackage{setspace}

\usepackage[utf8]{inputenc} 
\usepackage[T1]{fontenc}    
\usepackage{hyperref}       
\usepackage{url}            
\usepackage{booktabs}       
\usepackage{amsfonts}       
\usepackage{nicefrac}       
\usepackage{microtype}      
\usepackage{xcolor}         
\usepackage{hyperref}
\usepackage{appendix}
\usepackage{amsmath}
\usepackage{booktabs, tabularx}
\usepackage{code_styles}
\usepackage{amsthm}
\theoremstyle{plain}

\theoremstyle{definition}

\theoremstyle{remark}

\newcommand{\E}{\mathbb{E}}

\usepackage{graphicx,booktabs,array}
\usepackage{subcaption,graphicx,booktabs}
\usepackage{enumitem}
\newcolumntype{C}[1]{>{\centering\arraybackslash}m{#1}}

\usepackage[round]{natbib}

\begin{document}
\twocolumn[

\aistatstitle{Where the Score Lives: A Wavelet View of Diffusion}

\aistatsauthor{ Emma Finn \And Binxu Wang \And  T. Anderson Keller \And  Demba E. Ba }

\aistatsaddress{The Kempner Institute for the Study of Natural and Artificial Intelligence \\
Harvard University, Cambridge, MA 02138} ]
 
\begin{abstract}
  Score-based generative models have had remarkable success over the last decade in generating a diverse set of visually plausible images. 
  A variety of architectures including CNNs,  U-Nets, and Transformers have been used as the score-approximation network in such diffusion modeling; however, to date, relatively little is known about how these architectural choices impact generative behavior. In this work, to provide insight into this area, we propose an analytically solvable parameterization of the score function using an expansion in a 2D orthogonal wavelet basis. In particular, we derive interpretable optimal score functions in terms of the moments of the data distribution. We use this parametrization to provide an architecture-agnostic, moment-based analysis that reveals which attributes of the data distribution tend to matter most for denoising. Our score machine is flexible enough to partially mimic the relevant inductive biases of multiple architectures, including U-Nets and CNNs, taking a step towards understanding why different score architectures can exhibit distinct generative behavior. Since our score is solvable in terms of the moments of the data, we can begin to understand how the data distribution interacts with the score network to produce the behavior we observe in diffusion models.
\end{abstract}

\section{INTRODUCTION}

Diffusion models have rapidly advanced image generation and many other generative tasks in recent years \citep{sohl-dickstein2015deep, song2019generative, ho2020denoising, karras2022edm, lou2024discretediffusionmodelingestimating}. However, it is still not clear what causes their remarkable ability to construct visually coherent samples which often generalize beyond their training distribution. This generalization property clearly depends both on the score network used and on the underlying properties of the data distribution; since it is clear that a diffusion model trained on a single image would not generalize well.

Recent work has explored both architectural contributions to the kind of creativity displayed in diffusion \citep{kamb2024analytict} and data-determined contributions \citep{wang2024unreasonableeffectivenessgaussianscore}. 
On the architectural side \citet{kamb2024analytict} demonstrate that a CNN encodes certain inductive biases (namely translation equivariance) into a denoising network when used as a diffusion backbone, encouraging a certain `patch mosaic' form of creativity; one step further, \cite{wang2025analytical} shows that when the backbone is linear and translation-equivariant, the diffusion model learns a stationary Gaussian process version of the data; 
building on prior work \cite{li2024goodscoredoeslead} which demonstrates that under certain conditions, the ideal score simply memorizes the training data. 

\looseness=-1
On the data-distribution side, spatial locality has been shown to dominate denoising across time scales \citep{niedoba2025mechanisticexplanationdiffusionmodel}, and a first-order Gaussian score has additionally been shown to explain a significant amount of observed behavior \citep{wang2024unreasonableeffectivenessgaussianscore}, including the spatial local dependency of the denoiser \citep{lukoianov2025locality}. 
Generalization grows with dataset size \citep{bonnaire2025diffusionmodelsdontmemorize}, and learned scores can be interpreted as shrinkage in a geometry-adapted harmonic basis \citep{kadkhodaie2024generalizationdiffusionmodelsarises}. Yet, it still remains unclear which distributional statistics, and which score network architectural choices, are most critical to score learning -- questions central to understanding and improving diffusion models.

In this work, we propose an interpretable parameterization of the score function using an expansion in a wavelet basis, which is remarkably flexible. We then derive an ideal score function particular to our functional form, and implement a wavelet-based score machine. Empirically, we run  experiments across different families of denoisers to understand which components of the data distribution are most relevant in which settings.

\subsection{Background on Diffusion Models}
We first provide a short overview of score-based generative modeling. Diffusion models operate by corrupting all of the input data over time according to a noising process, given by an Ornstein-Uhlenbeck stochastic differential equation \citep{song2021scorebasedgenerativemodelingstochastic}:
\begin{equation}
    d \mathbf{X}_t = \underbrace{f(\mathbf{X}_t, t)}_{\text{drift}} dt + \underbrace{g(t)}_{\text{diffusion}} dW.
\end{equation}
The reverse process is given by: 
\begin{equation}
    d \mathbf{X}_t = \left[f(\mathbf{X}_t, t) - \frac{1}{2}g(t)^2 \underbrace {\nabla_{X_t} \log (p_t(X_t))}_{\text{score function}}\right] dt + g(t) d\bar{W}.
\end{equation}
We parametrize a neural network $s_{\theta}$ to approximate the score function, and in our setup, we use the score matching objective \citep{steinsLemmaSource}:
\begin{equation}
    \mathcal{L}(\theta) = \int_0^T \lambda(t) \mathbb{E}_{p_t(\mathbf{X})} [\|s(\mathbf{X},t) - {s}_{\theta}(\mathbf{X},t)\|^2_2]
\end{equation}
Where $\lambda$ is a time dependent weight. For simplicity, in our case, we set $\lambda(t) \equiv 1$. To further simplify, we'll divide our time interval $[0,1]$ into $N$ discrete steps, and treat each step independently, so the loss decouples across time. Our loss at any particular timestep is  
\begin{equation}
    \mathcal{L}^{(t)}(\theta) = \mathbb{E}_{p_t(\mathbf{X})} [\|s(\mathbf{X},t) - {s}_{\theta}(\mathbf{X},t)\|^2_2]
\end{equation}
and our overall loss is simply the sum of all such $\mathcal{L}^{(t)}$. In the following, we will demonstrate how to parametrize the score function instead in a wavelet basis, and thereby find the optimal parameters exactly in closed form. 
\subsection{Wavelets and Denoising} \label{sec:waveletsdenoise}
\begin{figure}[t] 
\vspace{-8pt}
\centering
\includegraphics[width=0.80\columnwidth]{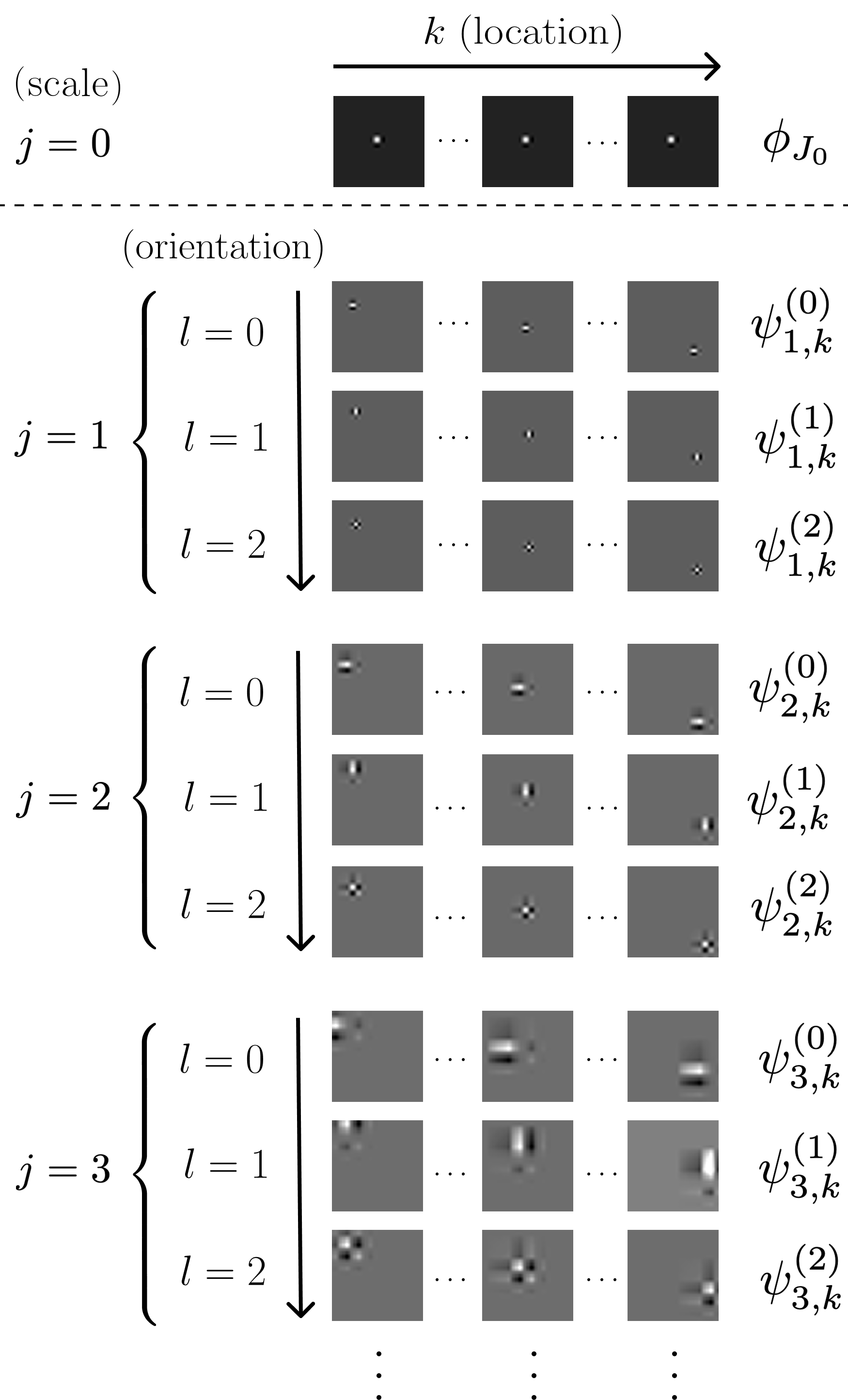}
\caption{\textbf{Daubechies Wavelets.} Visualized across scales $j$, locations $k$, and orientations $l$.}

\label{fig:onecol}
\end{figure}
Wavelets have long been used for image denoising and representation.  \cite{mallat1989multires} introduced multi-resolution analysis (MRA) and a fast algorithm for the discrete wavelet transform. Classical wavelet denoising by soft thresholding was developed by \cite{donoho_ideal_1994}, and further in \citep{Donoho1995SoftThresholding, daubechies1992ten, chang2000adaptive, portilla2003image}, with related analytic multiscale constructions such as wavelet scattering networks \citep{bruna2013invariant} and the steerable pyramid \citep{simoncelli1995steerable}. Like Fourier methods, wavelet decompositions represent square-integrable functions in a basis indexed by “frequencies,” but unlike Fourier bases, wavelets are localized in both space and scale, providing joint spatial–frequency resolution. 

More recently, \citet{phung2023waveletdiffusionmodelsfast} have attempted to use wavelets to speed up diffusion sampling and other score-based generative frameworks, as in \citep{guth2022waveletscorebasedgenerativemodeling}. In \citep{phung2023waveletdiffusionmodelsfast}, the authors implement a modified diffusion model which leverages wavelets to increase sampling and training efficiency. Their approach applies wavelet decomposition to the input image and performs denoising in wavelet space, reducing computational cost by separating low- and high-frequency components into a more structured representation than raw pixels. This wavelet transform is used not only at the image level, but also throughout the generator’s feature hierarchy. 

Moreover, there is an increasing awareness that wavelet approximations can be used to improve the quality of samples derived from modern image generation models, like diffusion models. For example in \citep{sigillo2026latentwavelet}, the authors demonstrate that wavelets can be used to significantly increase the quality of texture and detail in very high resolution images. In particular, they propose a ``wavelet-derived spatial saliency map'' which uses a discrete wavelet transformation to identify where the image stores most of its structural detail. This wavelet-based approach to improve detail quality for high-resolution images has also been proposed in \citep{Zhang_2025_CVPR}. Finally, outside the context of image generation, papers such as \citep{hu2025waveletdiffusionneuraloperator} have proposed using wavelet transforms to help diffusion-based methods for PDE simulation better capture sharp changes and ``long-term dependencies.'' As in \citep{phung2023waveletdiffusionmodelsfast}, the key insight is that diffusion is better behaved in the wavelet domain than in the original state space of the PDE.

This performance improvement across several different problem settings suggests that wavelets can be a useful representation for classes of data on which one might want to train diffusion models, especially when multi-resolution structure is important. Moreover, \citep{falck2023multiresolutionframeworkunetsapplications} show, in a multi-resolution framework, that U-Nets with average pooling implicitly learn a Haar wavelet basis representation of the data, which further motivates our use of wavelets to approximate the score functions that are typically learned by U-Nets.

We adopt compactly supported Daubechies wavelets (Fig. \ref{fig:onecol}). Let $\phi$ denote the scaling (“father”) function and $\psi$ the wavelet (“mother”) function in one dimension. In two dimensions, we form tensor products to obtain one scaling atom and three detail atoms at each location/scale. The translates and dilates of these functions form an orthogonal basis of the entirety of $L^2(\mathbb{R}^2)$. 

For scale $j\in\mathbb Z$ and translation $k=(k_1,k_2)\in\mathbb Z^2$:
\begin{align*}
\phi_{j,k}(\mathbf u) &= 2^{j}\,\phi(2^{j}u_1-k_1)\,\phi(2^{j}u_2-k_2),\\
\psi^{(\ell)}_{j,k}(\mathbf u) &= 2^{j}\,g^{(\ell)}(2^{j}u_1-k_1,\,2^{j}u_2-k_2),\qquad \ell\in\{0,1,2\},
\end{align*}
where $g^{(0)}(a,b)=\psi(a)\phi(b)$, $g^{(1)}(a,b)=\phi(a)\psi(b)$, and $g^{(2)}(a,b)=\psi(a)\psi(b)$. While the functions $\phi$ and $\psi$ are not available in closed form for the Daubechies wavelets, for integer $N\ge1$, the Daubechies–$N$ wavelet $\psi(x)$ and scaling function $\phi(x)$ are defined by the two‐scale relations:
\begin{align*}
\phi(x)
 & = \sqrt{2}\sum_{n=0}^{2N-1}h_n\,\phi(2x-n), \\
\qquad
\psi(x)
& = \sqrt{2}\sum_{n=0}^{2N-1}g_n\,\phi(2x-n).
\end{align*}

With periodic boundary handling on $[0,1]^2$ (our default), the set
\[
\mathcal B_{J_0}\;=\;\Big\{\phi_{J_0,k}\Big\}_{k}\;\cup\;\Big\{\psi^{(\ell)}_{j,k}\Big\}_{j\ge J_0,\;k,\;\ell}
\]
forms an orthonormal basis of $L^2([0,1]^2)$. (In our implementation the scaling atoms are taken at the finest decomposition level (J0 = 1) and detail atoms at levels 1 through Jmax; this dictionary is over-complete (M = 1564 atoms at 32×32, 5628 at 64×64) and forms an overcomplete frame rather than an orthonormal basis.) At each fixed \((j,k)\), the triple \(\big(\psi^{(1)}_{j,k},\psi^{(2)}_{j,k},\psi^{(3)}_{j,k}\big)\) are the three orientation components often called the “horizontal”, “vertical”, and “diagonal” details.
We refer to this per-location orientation triplet as the detail band at \((j,k)\).

\subsection{Our Contributions}
Concretely, this work offers four core contributions:
\begin{itemize}
\item \textbf{Analytic wavelet score.} We parameterize the score in an orthonormal wavelet basis and, via Stein’s identity, reduce each coefficient to a closed-form ridge least-squares estimate with moment-based right-hand sides.
\item \textbf{Structured dependencies.} We introduce three interpretable classes which serve to isolate the statistics that matter across noise scales.
\item \textbf{Empirical findings.} On MNIST ($32{\times}32$ \& $64{\times}64$) we find: higher polynomial degree lowers MSE; \emph{local} coupling produces the most reliable denoising improvements across $\sigma$; \emph{band-tying} sharpens edges qualitatively; benefits of wavelets increase with resolution.
\item \textbf{Against trained baselines.} Compared to CNN/U-Net denoisers, our models narrow the gap at low–moderate noise by explicitly encoding short-range, multiscale locality—without gradient-based training and with interpretable coefficients.
\end{itemize}

\begin{figure}[t] %
\centering
\includegraphics[width=\columnwidth]{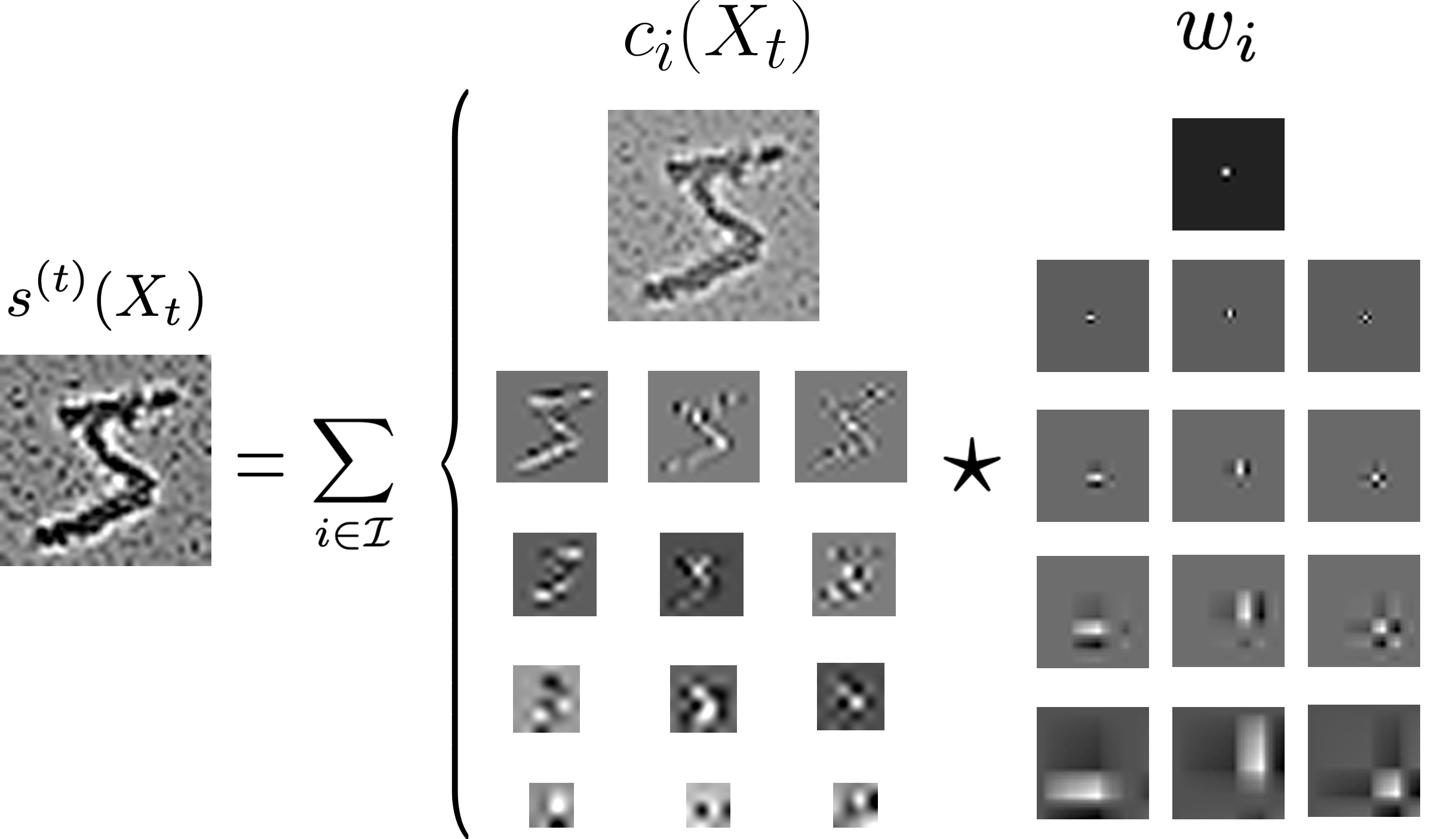}
\caption{\textbf{Wavelet Expansion of the Score}. The score of a given image $X_t$ can be expanded in an orthonormal wavelet basis by representing it as a sum of coefficients $c_i(X_t)$ multiplied by corresponding wavelets $w_i$ (Equation \ref{eqn:expanding_the_score}). Since the wavelets come in many translated copies, we only display wavelets at a single spatial location $k$, and display coefficients in spatially ordered `feature maps'. The sum over spatial locations is then implied by the convolution operator ($\star$).}
\label{fig:wavelet_expansion}
\end{figure}
\section{A WAVELET EXPANSION OF THE SCORE}
\subsection{Expanding the Score}
As motivated in Section~\ref{sec:waveletsdenoise}, wavelets provide a natural multiscale basis for representing diffusion scores. We now show that, after expanding the score in this basis and modeling each coefficient as a linear function of wavelet-derived features, score learning reduces to a collection of closed-form ridge regression problems.

Let $x \in \mathbb{R}^d$ denote a discretized version of an image, obtained by sampling the underlying grayscale images on  $[0,1]^2$ over the pixel grid. For each noise level $t$, let $X_t$ denote the corresponding noisy image. Let $\tilde{w}_i(u)$ be the $i^{th}$ continuous wavelet atom and let $w_i$ denote its sampled, vectorized version on the image grid. In the rest of the paper, unless otherwise noted, we choose to work in the discrete setting. Thus the score $s_{\text{true}}(x,t) = \nabla_x \log p_t(x) \in \mathbb{R}^d$. Recall that as a function of $x$, $s_{\text{true}}: \mathbb{R}^d \to \mathbb{R}^d$, but at a particular noisy image $X_t$, $s_{\text{true}}(X_t,t) \in \mathbb{R}^d$. 
For a grayscale image $X_t$ at noise level $t$ and score $s_{\text{true}}(\cdot, t)$, since wavelets form an orthonormal basis of $L^2$, the score at time t in a grayscale image can be expanded as
\begin{equation}
\label{eqn:expanding_the_score}
s^{(t)}(X_t)\;=\;\sum_{i\in\mathcal I}\,c_i(X_t)\,w_i,\quad
c_i(X_t)=\big\langle s_{\text{true}}(X_t,t),\,w_i\big\rangle,
\end{equation}
where $\{w_i\}_{i\in\mathcal I}=\mathcal B_{J_0}$ indexes $(J_0,k)$ for scaling atoms and $(j,k,\ell)$ for detail atoms (visualized in Figure \ref{fig:wavelet_expansion}). 
The way in which we choose to model $\langle  s_{\text{true}}(X_t), w_i\rangle$ determines the properties of the score function estimator. We model each coefficient by features $\varphi_i(X_t)\in\mathbb R^{d_i}$ and parameters $\alpha_i^{(t)}\in\mathbb R^{d_i}$ (visualized in Fig. \ref{fig:coeff_approx}):
\begin{equation}
\label{eq:model}
\widehat c_i(X_t)\;=\;\alpha_i^{(t)\top}\varphi_i(X_t),\quad
s_{\theta}^{(t)}(X_t)\;=\;\sum_{i\in\mathcal I}\widehat c_i(X_t)\,w_i .
\end{equation}
Because $\{w_i\}$ is orthonormal, the population squared loss decouples across $i$ and so taking the gradient w.r.t. $\alpha_i^{(t)}$ yields the simultaneous equations
\begin{align}
\label{eq:normalequations}
    &\mathcal L^{(t)}(\theta)= \mathbb E_{X_t\sim p_t}\Big\|s(X_t)-\sum_{i}\alpha_i^{(t)\top}\varphi_i(X_t)w_i\Big\|_{L^2}^2 \\ 
    & \!\Rightarrow\!
\frac{\partial \mathcal L^{(t)}}{\partial \alpha_i^{(t)}}=-2\mathbb E\Big[\varphi_i(X_t)\big(\langle s(X_t), w_i \rangle-\alpha_i^{(t)\top}\varphi_i(X_t)\big)\Big]=0
\end{align}
Solving for the optimal coefficients $\alpha_i^{(t)*}$, we find, if $\Sigma_i = \mathbb{E}[\varphi_i\varphi_i^\top]$ is invertible:
\begin{align}
\label{eq:normal-eq-pt2}
\underbrace{\mathbb E[\varphi_i \varphi_i^\top]}_{\Sigma_i}\,\alpha_i^{(t)}
  &= \mathbb E\!\big[\varphi_i(X_t)\,\langle s(X_t), w_i \rangle\big] \\
\Rightarrow\quad
\alpha_i^{(t)\star}
  &= \Sigma_i^{-1}\,\mathbb E\!\big[\varphi_i(X_t)\,\langle s(X_t), w_i \rangle\big]
\end{align}
We can further simplify, by applying a more general form of Stein's Identity as in \citep{steinsLemmaSource} to the expectation of the inner product. In particular, for some function $f$, Stein's Score Identity says
\begin{equation}
\mathbb{E}_{X\sim p_t}\!\big[\,s_t(X)\cdot f(X)\,\big]
\;=\;
-\,\mathbb{E}_{X\sim p_t}\!\big[\,\nabla\!\cdot f(X)\,\big]
\end{equation}
for $p_t$ smooth and under vanishing boundary flux (e.g., periodic boundaries). 
In our case, we write $\varphi_i(X_t)=[\varphi_{i,1}(X_t),\dots,\varphi_{i,d_i}(X_t)]^\top$.
For each component $j$, apply Stein’s identity with the vector field
\[
f_j(x)\;=\;\varphi_{i,j}(x)\,w_i,\qquad \|w_i\|_2=1.
\]
Then
\begingroup\small
\begin{align}
\mathbb{E}\!\big[\varphi_{i,j}(X_t)\,\langle s_t(X_t), w_i\rangle\big]
  &= -\,\mathbb{E}\!\big[\nabla_{X_t}\big(\varphi_{i,j}(X_t)\,w_i\big)\big] \\
  &= -\,\mathbb{E}\!\big[\langle \nabla_{X_t} \varphi_{i,j}(X_t),\, w_i\rangle\big]
\end{align}
\endgroup

Stacking over $j=1,\dots,d_i$ gives the vector form
\begin{equation}
\label{eq:stein-vector}
\mathbb{E}\!\big[\varphi_i(X_t)\,\langle s_t(X_t),w_i\rangle\big]
\;=\;
-\,\mathbb{E}\!\big[(\nabla \varphi_i(X_t))^\top w_i\big]
\end{equation}
Thus, we find that our solution is 

\begin{equation}
\label{eq:solution_normal}
 \alpha_i^{(t)\star}\;=\;-\E[\varphi_i\varphi_i^\top]^{-1}\,\mathbb\,\mathbb{E}\!\big[(\nabla \varphi_i(X_t))^\top w_i\big]
\end{equation}

In practice, expectations are replaced by sample averages over $n$ training images at time $t$,
$\mathbb{E}[f(X_t)]\approx \frac{1}{n}\sum_{r=1}^n f\!\big(X_t^{(r)}\big)$. We also add a ridge regularization term, because $\hat{\Sigma}_i$ can be singular or ill-conditioned with high-degree features or correlated wavelet coefficients, we stabilize and control variance by adding a ridge regularization, as follows
\[
\hat{\alpha}_i^{(t)} \;=\; \big(\hat{\Sigma}_i+\gamma I\big)^{-1}\hat{b}_i,\quad \gamma>0,
\]
which guarantees an invertible system and mitigates overfitting.

Thus, for samples $\{X_t^{(n)}\}_{n=1}^N$ and ridge $\gamma\ge 0$,
\begin{align}
\label{eq:ridge}
\hat{\alpha}_i^{(t)}(\gamma)
= \bigg(& \frac{1}{N}\sum_{n=1}^N \varphi_i^{(n)} \varphi_i^{(n)\top} + \gamma I\bigg)^{-1} \nonumber \\
& \cdot\  \bigg(-\,\frac{1}{N}\sum_{n=1}^N (\nabla \varphi_i(X_t^{(n)}))^\top w_i\bigg).
\end{align}
$\nabla \varphi_i(X_t^{(n)})$ can be estimated analytically from the chosen features using the method of moments.

Diagonalizing $\Sigma_i=U\Lambda U^\top$ shows that \((\Sigma_i+\gamma I)^{-1}\) weights eigen-directions by $1/(\lambda+\gamma)$. Thus, given that the features are well suited to the data such that $\Sigma_i$ in \eqref{eq:normal-eq-pt2} is well conditioned and its columns are not highly correlated, the estimator weights eigen directions by $\frac{1}{\lambda + \gamma}$. Small-$\lambda$ directions receive higher coefficients but are more noise sensitive, which the ridge regression helps to limit. Intuitively, what this says is that this score approximator emphasizes lower-variance feature directions (small~$\lambda$) and down weights higher-variance directions (large~$\lambda$). This aligns with the observation that natural images exhibit approximately power-law spectral decay and sparse wavelet coefficients, whereas white noise spreads energy more uniformly; consequently, informative structure often concentrates in a subset of scales and orientations. The model learns to correct more strongly along low‐variance modes of the data distribution and ignore the high‐variance ones. 

\begin{figure}[t] %
\centering
\includegraphics[width=.99\columnwidth]{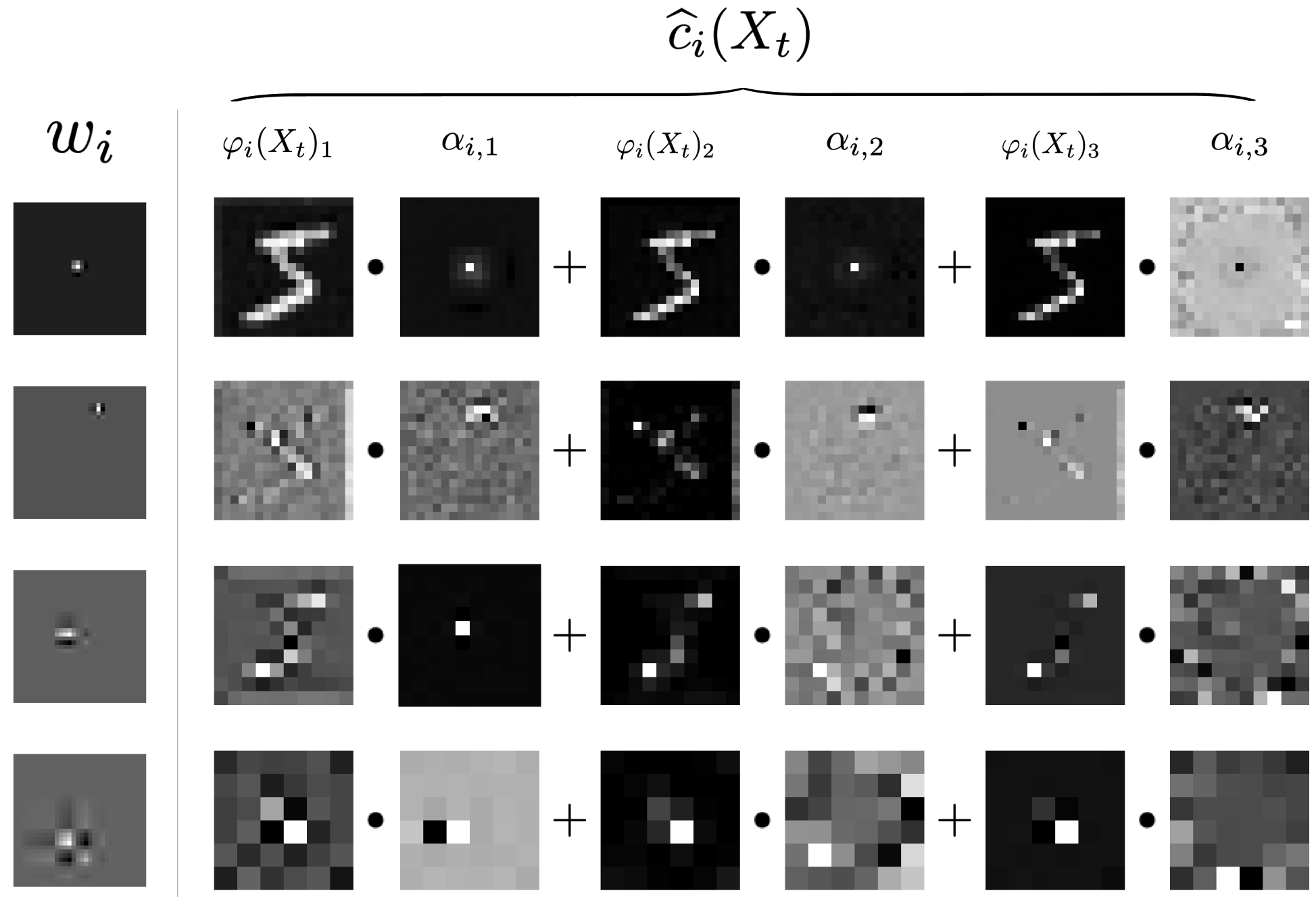}
\caption{\textbf{Wavelet Coefficient Approximation}. We approximate the unknown true coefficients $c_i(X_t)$ of the true score function as an inner product of features $\varphi_i(X_t)$ and parameters $\alpha_i^{(t)}$. The features depicted here are \emph{independent} degree 3 polynomial features.}
\label{fig:coeff_approx}
\end{figure}

\subsection{Correlation Structures in the Data}
Natural images exhibit structured dependencies in the wavelet domain: heavy–tailed marginals per coefficient, co-activation across orientations at a fixed location, and spatial persistence along edges and textures \citep{NIPS1999_6a5dfac4}. As the noise level decreases, these dependencies become more pronounced. We therefore study three families that isolate, then incrementally reintroduce, these correlations.

\begin{enumerate}[label=(\roman*),leftmargin=*]
\item \emph{Independent (diagonal).} We model each coefficient
\(y_t^{(i)}=\langle X_t,w_i\rangle\) in isolation using degree–\(D\) monomials (or probabilists’ Hermite polynomials for numerical stability).
This “mean-field” score approximator is the right baseline when \(p_t\) is close to a product distribution (early time steps, near-Gaussian), and it makes failures highly interpretable: any gap to stronger models directly measures the predictive value of cross-coefficient structure that a diagonal model cannot use.

\item \emph{Band-tied.} At a fixed scale and location \((j,k)\), we couple the three detail orientations \(\ell\in\{0,1,2\}\) with degree–\(D\) interactions (each monomial includes the target coordinate).
This targets cross-orientation co-activation caused by edges and corners and mirrors the channel-mixing inductive bias of CNN/U-Net score networks at a single spatial site.
Improvements here isolate the contribution of  within-pixel, within-scale structure—testing whether local orientation interactions alone explain denoising gains as \(t\) decreases.

\item \emph{Local-coupled.} At fixed scale and orientation, we allow the coefficient at location \(k\) to depend on neighbors \(k+\delta\) within a Chebyshev ball \(\|\delta\|_\infty\le r\).
This realizes a small neighborhood in wavelet space and emulates the increasing receptive field of convolutional (or locally attentive) score networks. Gains that grow then saturate with \(r\) would quantify how much spatial context the score actually needs for accurate denoising and globally consistent structure as noise recedes.

\end{enumerate}

We use \eqref{eq:stein-vector} to build families of models that probe how scale/orientation structure in the data influences the score. We propose three families of models which can provide insight into what properties of the data distribution are relevant for score-based diffusion.
In the Daubechies wavelet setting, we have 
\begin{align}
s_\theta(X_t)^{(t)}(\mathbf{u}) = 
\sum_{k \text{ ranging over J0 scaling grid}} \theta_k(X_t, t) \, \phi_{J_0, k}(\mathbf{u}) \nonumber \\
+ \sum_{j = J_0}^{J_{\max}} \sum_{k \text{ ranging over level j detail grid}} \sum_{l=0}^2 \zeta_{j,k,l}(X_t, t) \, \psi_{j,k}^{(l)}(\mathbf{u}).
\end{align}

\subsubsection{Independent Baseline}
Assume (unrealistically) that coefficients decouple across scale $j$, location $k$, and detail band $l$. Let $y_t^{(i)}=\langle X_t,w_i\rangle$ with moments
$\mu_{r}^{(i)}(t)\;=\;\mathbb E\!\big[(y_t^{(i)})^r\big]$
and model our coefficients as
\begin{align}
\theta_k(X_t,t)& =\sum_{m=0}^D b_m^{(k)}(t)\,\langle X_t,\phi_{J_0,k}\rangle^m,
\\
\zeta_{j,k,\ell}(X_t,t)& =\sum_{m=0}^D d_m^{(j,k,\ell)}(t)\,\langle X_t,\psi^{(\ell)}_{j,k}\rangle^m.
\end{align}

Using \eqref{eq:stein-vector} with monomial features gives the normal equations analogous to~\eqref{eq:normalequations} for each index $i$ and $r=0,\dots,D$. In particular we substitute $\phi(y) = y^m$ into Stein's identity and apply it in each coefficients to give
\begin{equation*}
\sum_{m=0}^D a_m^{(i)}(t)\mu_{m+r}^{(i)}(t)=-r\mu_{r-1}^{(i)}(t),
\;\ a_\bullet^{(i)}\in\{b_\bullet^{(k)},d_\bullet^{(j,k,\ell)}\},
\end{equation*}
This is a Hankel system $H^{(i)}(t)\,a^{(i)}(t)=-\,h^{(i)}(t)$, where
\begin{align*}
H^{(i)}_{r,m}(t) & =\mu_{r+m}^{(i)}(t), \\
h^{(i)}(t) & =\big[\,0,\ \mu_{0}^{(i)}(t),\ 2\mu_{1}^{(i)}(t),\dots, D\mu_{D-1}^{(i)}(t)\,\big]^\top.
\end{align*}
Observe that all entries of \(H^{(j,k,l)}(t)\) and \(h^{(j,k,l)}(t)\) are computable from clean-data raw moments of \(Y_0=\langle X_0,\psi^{(l)}_{j,k}\rangle\).
\begingroup\small
\begin{equation}
\label{eq:moment-binomial}
\mu_{r}^{(i)}(t)
=\sum_{m=0}^{r} \binom{r}{m}\,\bar\alpha_t^{\frac{m}{2}}\,(1-\bar\alpha_t)^{\frac{r-m}{2}}
\mathbb E\!\big[\langle X_0,w_i\rangle^{m}\big]
\mathbb E[Z^{r-m}].
\end{equation}
\endgroup
where $\bar{\alpha}$ is the same $\bar{\alpha_t}$ derived from the noise schedule $\{\alpha_t\}_{t}$ found in the standard DDPM variance-preserving forward process $X_t = \sqrt{\bar{\alpha_t}} X_0 + \sqrt{1 - \bar{\alpha_t}}Z$ for $Z$ a standard multivariate normal random variable. We 
For small $D$, we can easily compute $a^{(i)}(t)=-(H^{(i)}(t))^{-1}h^{(i)}(t)$ (or $(H^{(i)}+\gamma I)^{-1}$ with ridge $\gamma\!\ge\!0$). The co-factor formula for the coefficients reads
\begin{equation}
    \hat{\alpha}_{i}(t)
\;=\;
-\,\frac{1}{\det H^{(i)}(t)}\,
\sum_{r=1}^{D} r\,\mu^{(i)}_{r-1}(t)\;C^{(i)}_{r,m}(t)
\end{equation}
with \(C^{(j,k,\ell)}_{r,m}(t)\) the \((r,m)\)-cofactor of \(H^{(i)}(t)\). This closed form can be used to investigate how higher order moments of the data distribution impact the score. Though clearly independence is an unrealistic assumption, this model is fast and interpretable, and serves as (i) an initial baseline, and (ii) a diagnostic lower bound. We can investigate the value of different kinds of correlation by measuring the difference in performance between models with different kinds of non-linear features.
Allowing each coefficient to depend arbitrarily on  degree–$D$ multivariate polynomials leads to
combinatorial parameter growth and brittle estimation. We therefore restrict to
structured, computationally tractable, and interpretable non-linear features
(diagonal/independent, band-tied, and local-coupled).

\subsubsection{Wavelet Band Coupling}
One very natural form of correlation is to allow wavelet coefficients at the same scale $j$ and location $k$ to depend on one another. Let 
\(y_0 = \langle X_t, \psi_{j,k}^{(0)} \rangle, \quad 
y_1 = \langle X_t, \psi_{j,k}^{(1)} \rangle, \quad 
y_2 = \langle X_t, \psi_{j,k}^{(2)} \rangle, \)
At each scale/location in the detail bands $(j,k,l)$, we assume a polynomial coefficient form of degree $D$:
\begin{equation}
\begin{aligned}
\zeta^{(t)}_{j,k,0}(X_t)
&= C^{(t)}_0 
+ \sum_{a=1}^{D} \sum_{b=0}^{D-a} \sum_{c=0}^{D-a-b} 
    \beta^{(0)}_{a,b,c}\, y_0^{a} y_1^{b} y_2^{c},\\
\zeta^{(t)}_{j,k,1}(X_t)
&= C^{(t)}_1 
+ \sum_{b=1}^{D} \sum_{a=0}^{D-b} \sum_{c=0}^{D-b-a} 
    \beta^{(1)}_{a,b,c}\, y_0^{a} y_1^{b} y_2^{c},\\
\zeta^{(t)}_{j,k,2}(X_t)
&= C^{(t)}_2 
+ \sum_{c=1}^{D} \sum_{a=0}^{D-c} \sum_{b=0}^{D-c-a} 
    \beta^{(2)}_{a,b,c}\, y_0^{a} y_1^{b} y_2^{c}.
\end{aligned}
\end{equation}
The constraint ensures each monomial contains the target coordinate $y_\ell$ at least once; cross
terms are allowed but pure “other-orientation” terms are excluded. 
We define our $\theta$ coefficients as polynomials with the same optimal values as in the independent case. Estimation proceeds via the same normal-equation machinery as in \eqref{eq:stein-vector}, but with mixed moments at $(j,k)$:
\[
\mu^{(j,k)}_{pqr}(t)
:=\E\!\big[\,y_0^{p}y_1^{q}y_2^{r}\,\big],\qquad p+q+r\le D+1,
\]
which, under the forward process with
orthonormal $\{\psi^{(\ell)}_{j,k}\}$, expand into clean-data mixed moments and factorized Gaussian
noise moments.

\subsubsection{Local Coupling}
Another natural choice of coupling is to allow wavelets in the same local neighborhood. For a radius $r\in\mathbb N$, define the neighborhood
\(
 \Delta_r \;=\; \{\delta\in\mathbb Z^2:\ \|\delta\|_\infty \le r,\ \delta\neq (0,0)\}
\)
 We allow wavelets in the same neighborhood to interact. Fixing a scale $j$ and orientation $\ell\in\{0,1,2\}$, let the oriented wavelet/detail coefficient at spatial location $k\in\mathcal K$ be
\(y_k \;=\; \langle X_t,\, \psi^{(\ell)}_{j,k}\rangle\). 
Let \(D\ge 1\) be the total degree. Define
\(S_D \;=\; \{(d,e)\in\mathbb N^2:\ d,e\ge 1,\ d+e\le D\} \).
We define the functional forms of $\theta$ as follows:
\begin{equation*}
\begin{alignedat}{1}
&\theta_{J_0,k}(X_t,t)
=\; \sum_{i=0}^{D} \alpha_i\, \big\langle X_t,\ \phi_{J_0,k}\big\rangle^{i} \\
&\quad+\; 
\sum_{\delta\in\Delta_r}\ 
\sum_{\substack{(d,e)\in S_D}}
\beta_{\delta,d,e}\, \big\langle X_t,\ \phi_{J_0,k}\big\rangle^{\,d}\, \big\langle X_t,\ \phi_{J_0,k + \delta}\big\rangle^{\,e}
\end{alignedat}
\end{equation*}
where $\alpha_i$ and $\beta_{\delta, d,e}$ are the parameters over which we optimize. Similarly for $\zeta$ we  define
\begin{equation*}
\begin{alignedat}{1}
&\zeta_{j, k, l}(X_t, t)
=\; \sum_{i=0}^{D} \xi_i\, \big\langle X_t, \psi_{j,k,l}\big\rangle^{i} \\
&\quad+\;
\sum_{\delta\in\Delta_r}\ 
\sum_{\substack{(d,e)\in S_D}}
\omega_{\delta,d,e}\, \big\langle X_t, \psi_{j,k,l}\big\rangle^{\,d}\, \big\langle X_t, \psi_{j,k + \delta,l}\big\rangle^{\,e}  
\end{alignedat}
\end{equation*}
where $\xi_i$ and $\omega_{\delta, d,e}$ are the parameters over which we optimize. 

\section{METHODS}
In order to implement this we return to the most general framing of the problem: ridge-regression in a non-linear feature space defined by wavelet coefficients. 
\paragraph{Preprocessing} We begin by standardizing the images and resizing them to $32 \times 32$ or $64 \times 64$. We use compactly supported Daubechies (db2) wavelets with periodic padding. The 2-D tensor-product basis $\mathcal B_{J_0}$ includes scale and detail atoms up to the deepest valid level for the filter (here $J_{\max}=3$ at $32\times 32$ in order to avoid boundary effects) with three orientations $\ell\in\{0,1,2\}$. The resulting basis forms a linear operator, $B \in \mathbb{R}^{M \times d}$ where $M$ is the number of atoms and $d = H \times W$ is the flattened dimension of the image. $B$ maps vectorized images $X_t \in \mathbb{R}^d$ to wavelet coefficients.

\begin{figure}
    \centering
    \includegraphics[width=\linewidth]{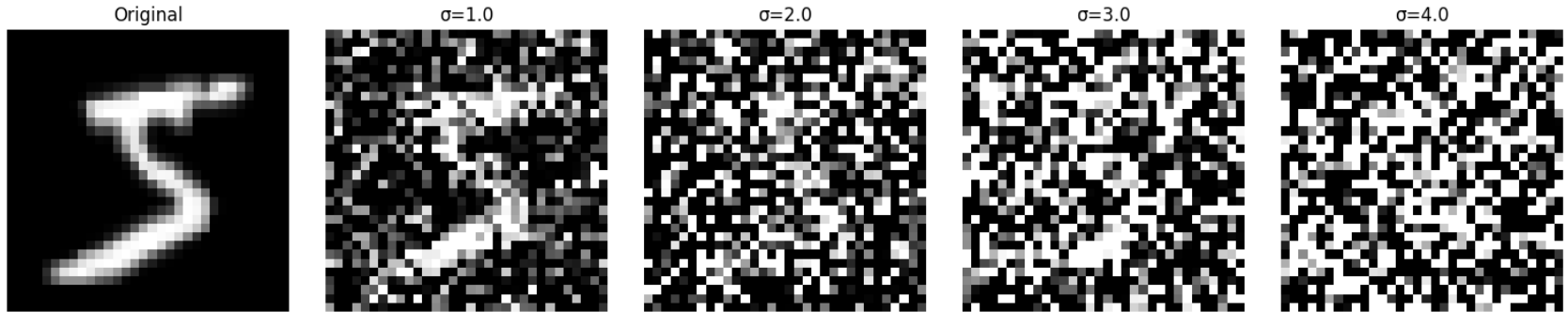}
    \caption{\textbf{Visualization of noise regimes.}  Clean image $\sigma = 0$ (left) to $\sigma =4$ (right).}
    \label{fig:noisyexamples}
\end{figure}
\paragraph{Noise Model} We add Gaussian noise according to the variance exploding regime $X_t = X_0 + \sigma Z$ for $Z \sim \mathcal{N}(0,1)$. Because the images have unit variance after standardization, $\sigma$ is directly interpretable as a noise-to-signal ratio. We consider four noise-to-signal regimes in the below experiments, $\sigma \in \{1,2,3,4\}$ (see Fig. ~\ref{fig:noisyexamples}). We note that in our theoretical derivation~\ref{eq:moment-binomial}, we use the variance preserving formulation of the noising schedule. An equivalent form in terms of $\sigma_t$ is also possible and just requires altering the coefficients to be in terms of the function of $\sigma_t$ corresponding to $\bar{\alpha}_t$ used there. This yields noise projections $c = X_t B^T$ that serve as inputs to our non-linear feature constructions.

\paragraph{Feature Expansion and Regression} We generate the non-linear features according to the model of correlation we're interested in. In practice we solve the ridge regularized optimization 
\begin{equation}
    \widehat W_\gamma \;=\; \arg\min_W \frac1N\|AW-Y\|_F^2 + \gamma\|W\|_F^2
\end{equation}

where $A$ is a matrix containing the stacked features $\varphi(X_t)$ computed on a finite number of noised samples $X_t$ with  shape $N \times p$, $Y$ is the matrix of corresponding clean images whose $n^{th}$ row is $X_0^{(n)}$, and $\gamma = 10^{-2}$ is the ridge penalty used in practice. Note that this pooled formulation is not identical to the per-coefficient models of Section 2.2: every output pixel is a linear function of the full feature dictionary, so the families differ only in which nonlinear interaction features the dictionary contains.

\textbf{Lemma} (Informal) The optimal $W$ from this equation yields the same denoised predictions as fitting each wavelet coefficient by ridge regression on the full feature dictionary, Equation ~\ref{eq:solution_normal}.

\textbf{Lemma} (Formal) Let $d = H \times W$ and $B \in \mathbb{R}^{M \times d}$ be a dictionary with unit norm rows $w_i^T$ consisting of wavelet atoms. Assume that $\text{rank(B) = d}$. Let $F:\mathbb{R}^M\to\mathbb{R}^P$ be any fixed feature map, writing $\Phi(x) := F(xB^\top)$ and denoting the stacked features as $\Phi$. Then, we have that the two ridge problems yield the same optima. 
\begingroup\small
\begin{equation}
\nonumber
    \text{\bf(Pixel)}\;\
W_\gamma^\star \in \arg\min_{W\in\mathbb{R}^{P\times d}}
\; \mathbb{E}\bigl[\|\Phi W - X_0\|^2\bigr] + \gamma\|W\|_F^2,
\end{equation}
\endgroup
\begingroup\small
\begin{equation}
    \text{\bf(Coeff)}\;\
U_\gamma^\star \in \arg\min_{U\in\mathbb{R}^{P\times M}}
\; \mathbb{E}\bigl[\|\Phi U - X_0B^\top\|^2\bigr] + \gamma\|U\|_F^2.
\end{equation}
\endgroup
Then the minimizers satisfy $U^\star \;=\; W^\star B^\top$ and for any left inverse $R \in \mathbb{R}^{d \times M}$ of $B$, the denoised predictions are the same. 
We defer the proof of this equivalence to Appendix ~\ref{app:math}. 

We also add a ridge term to improve numerical stability \citep{hoerl1970ridge}. We additionally experiment with the probabilists’ Hermite polynomials to expand our $\theta$ and $\zeta$ with increased stability. 
The probabilists’ Hermite polynomials $\{\mathrm{He}_n(x)\}_{n\ge0}$ are defined by $\mathrm{He}_n(x)=(-1)^n e^{x^2/2}\frac{d^n}{dx^n}e^{-x^2/2}$ and are orthogonal for $Z\sim\mathcal N(0,1)$ with $\mathbb E[\mathrm{He}_m(Z)\mathrm{He}_n(Z)]=n!\,\delta_{mn}$. 
Using $\mathrm{He}_n$ in place of raw monomials makes feature vectors orthogonal in expectation under near-Gaussian coordinates $y_t^{(i)}$, reducing covariance and bringing $\Sigma_i=\mathbb E[\varphi_i\varphi_i^\top]$ closer to diagonal (hence a smaller condition number). In practice, this improves the conditioning of the matrix equations \eqref{eq:normal-eq-pt2} and yields coefficients less sensitive to scaling and polynomial degree, especially when combined with a small ridge penalty. The results for this model are deferred to Appendix ~\ref{app:results}. 

\paragraph{Denoising and Comparison with Models} We denoise our images according to $\widehat{X_0} = A \widehat{W}$ map back to the original $[0,1]$ pixel scale, clamp the prediction to $[0,1]$, and report MSE against the clean images over the full dataset (95\% CIs over images). We also compare our reconstruction MSE to that of trained diffusion models with U-Net and CNN backbones. We defer the details of training those models to Appendix~\ref{app:training}.

\section{EXPERIMENTS AND DISCUSSION}

With these closed-form equations and model families in hand, we now test how much band and local coupling improve over the independent baseline across different noise levels. These models are just a few examples of the correlation structures in the data that can be probed with this approach.

\begin{figure}
    \centering
    \begin{subfigure}{0.92\linewidth}
        \centering
        \includegraphics[width= 0.9\linewidth]{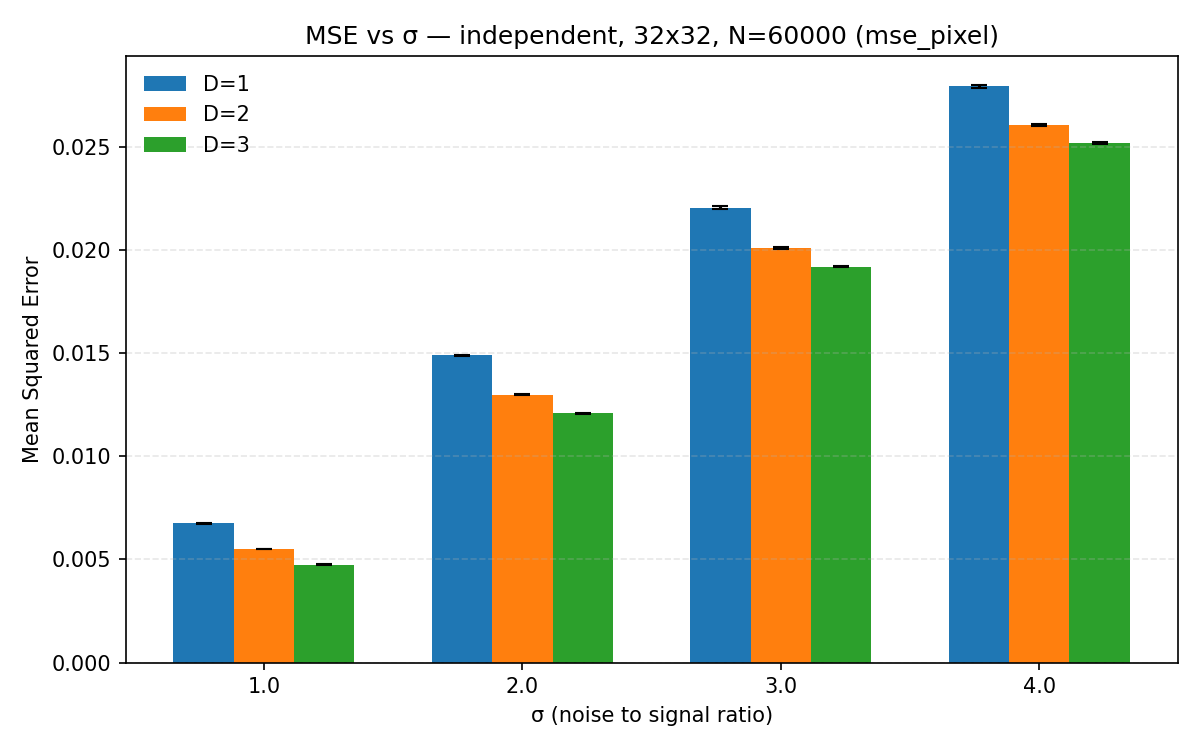}
        \caption{Independent Features}
        \label{fig:indep-32}
    \end{subfigure}

    \begin{subfigure}{0.92\linewidth}
        \centering
        \includegraphics[width= 0.9\linewidth]{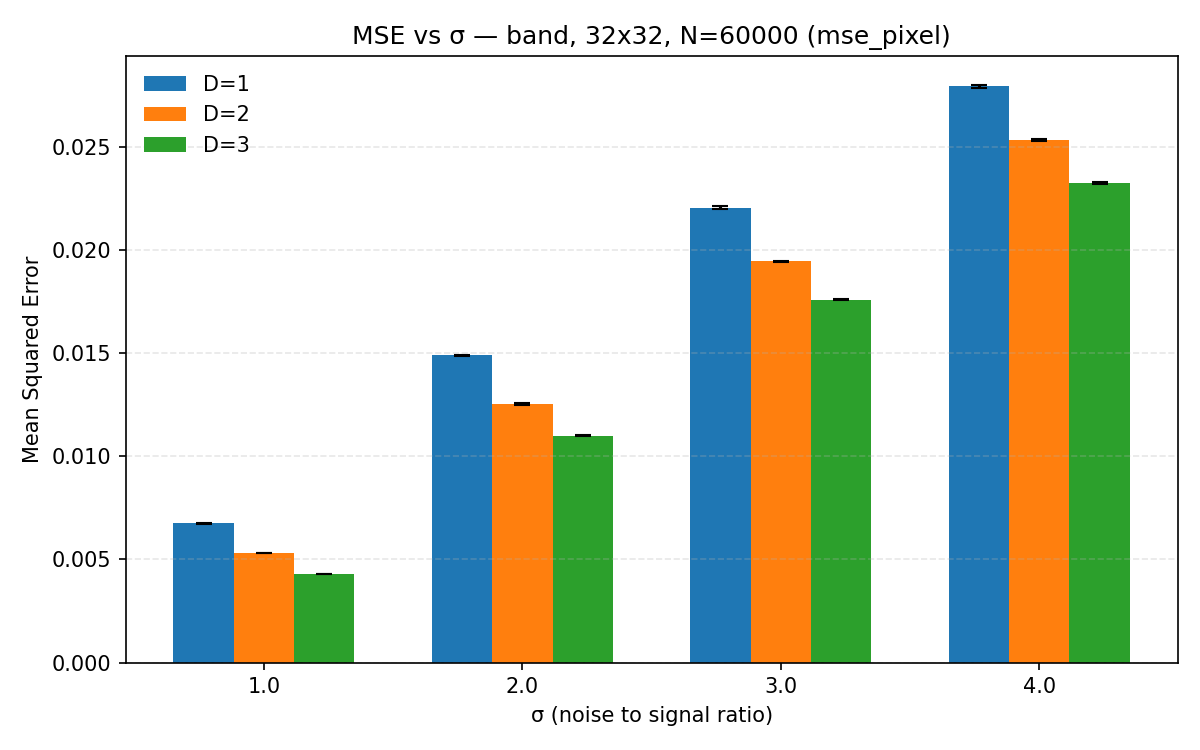}
        \caption{Band-Tied Coupling}
        \label{fig:band-32}
    \end{subfigure}

    \begin{subfigure}{0.92\linewidth}
        \centering
        \includegraphics[width= 0.9\linewidth]{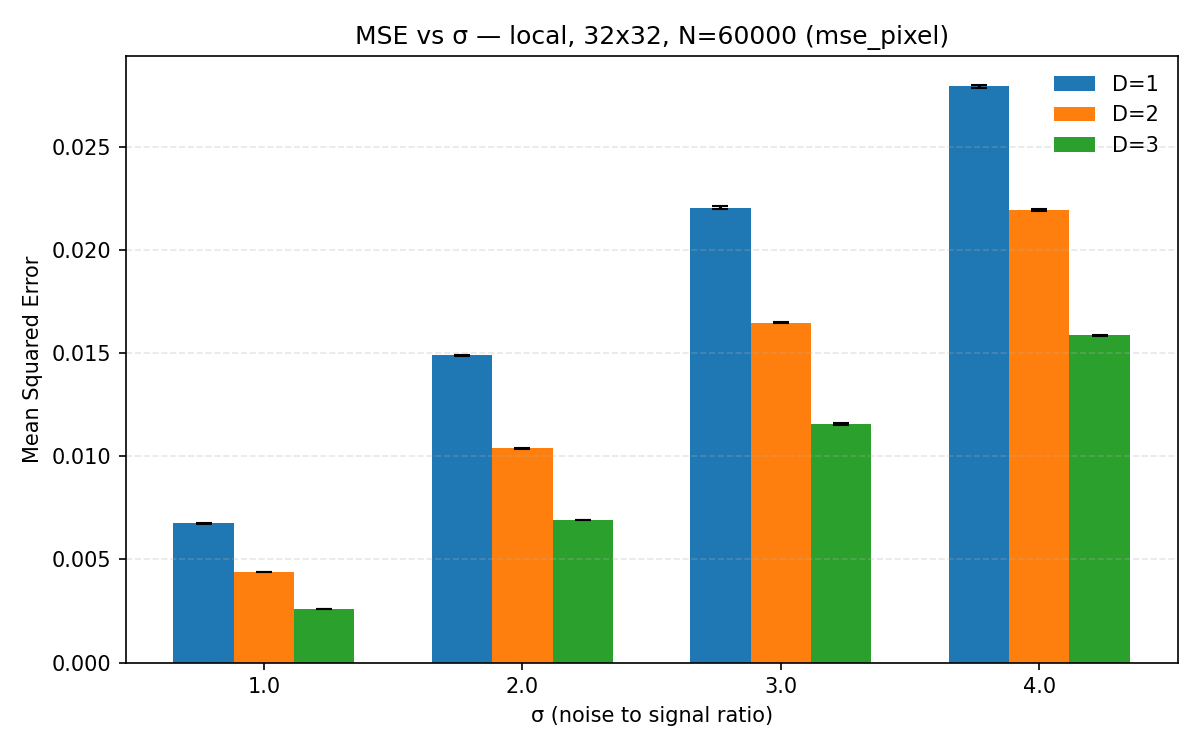}
        \caption{Local Coupling}
        \label{fig:local-32}
    \end{subfigure}

    \begin{subfigure}{0.96\linewidth}
        \centering
        \includegraphics[width= 0.99\linewidth]{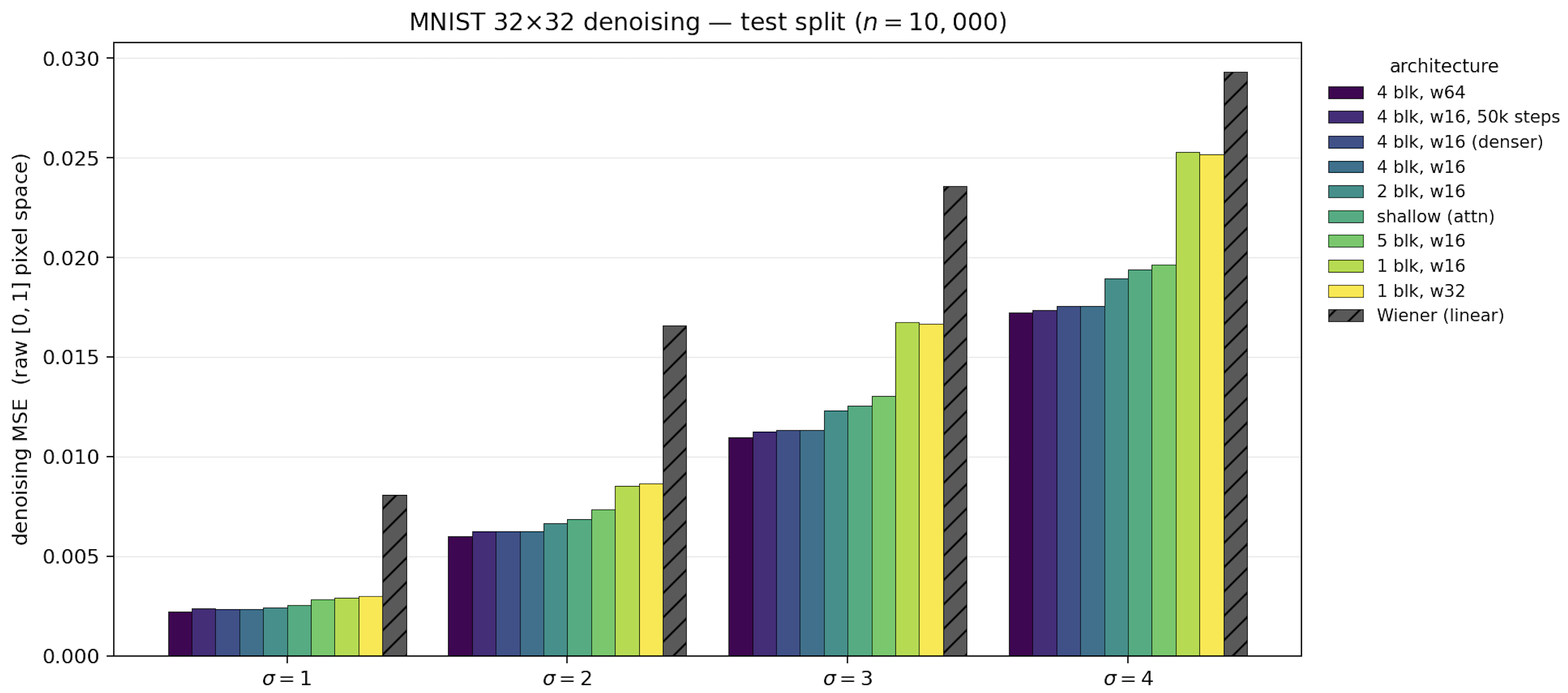}
        \caption{Trained Diffusion Models}
        \label{fig:models-32}
    \end{subfigure}
    
    \caption{\textbf{MNIST-32.} \textbf{(a)} Denoising MSE with independent monomial features across three sets of features: degrees 1, 2, \& 3. \textbf{(b)} Same features with band-tied coupling, and \textbf{(c)} local coupling. \textbf{(d)} Comparable performance of a variety of trained U-Net score function approximators.}
    \label{fig:32}
\end{figure}
\paragraph{Setup} We evaluate all models on two copies of the full MNIST training set of $60{,}000$ images, one set resized to $32 \times 32$ and the other resized to $64 \times 64$, denoted MNIST-32 and MNIST-64 respectively \citep{lecun_mnist}.

\paragraph{Changing the Degree of the Approximation} 
As shown in Figure~\ref{fig:32} and Figure~\ref{fig:64}, increasing the polynomial degree consistently improves reconstruction error by enhancing the expressivity of the feature space. The gains are especially pronounced for correlated models—most notably in the local coupling case (Figure~\ref{fig:local-32})—indicating that higher-order moments of the wavelet-projected data are particularly valuable when modeling correlation structure. Strikingly, this effect is amplified at larger image resolutions: comparing Figure~\ref{fig:32} with Figure~\ref{fig:64} suggests that our wavelet-based reconstruction becomes more robust as image size increases.

\paragraph{Detail Band Correlation} As shown in Figure~\ref{fig:32} and Figure~\ref{fig:64}, coupling the three orientations at a location improves visual quality (see Appendix~\ref{app:results}). Qualitatively, edges and corners are sharper under band-tying, consistent with the hypothesis that cross-orientation co-activation carries some local predictive signal and is especially sensitive to sharp changes. This detail-band correlation can be viewed as a non-linear cross-channel interaction (e.g., quadratic products within a scale).

\begin{figure}
    \centering
    \begin{subfigure}{0.8\linewidth}
        \centering
        \includegraphics[width= 0.9\linewidth]{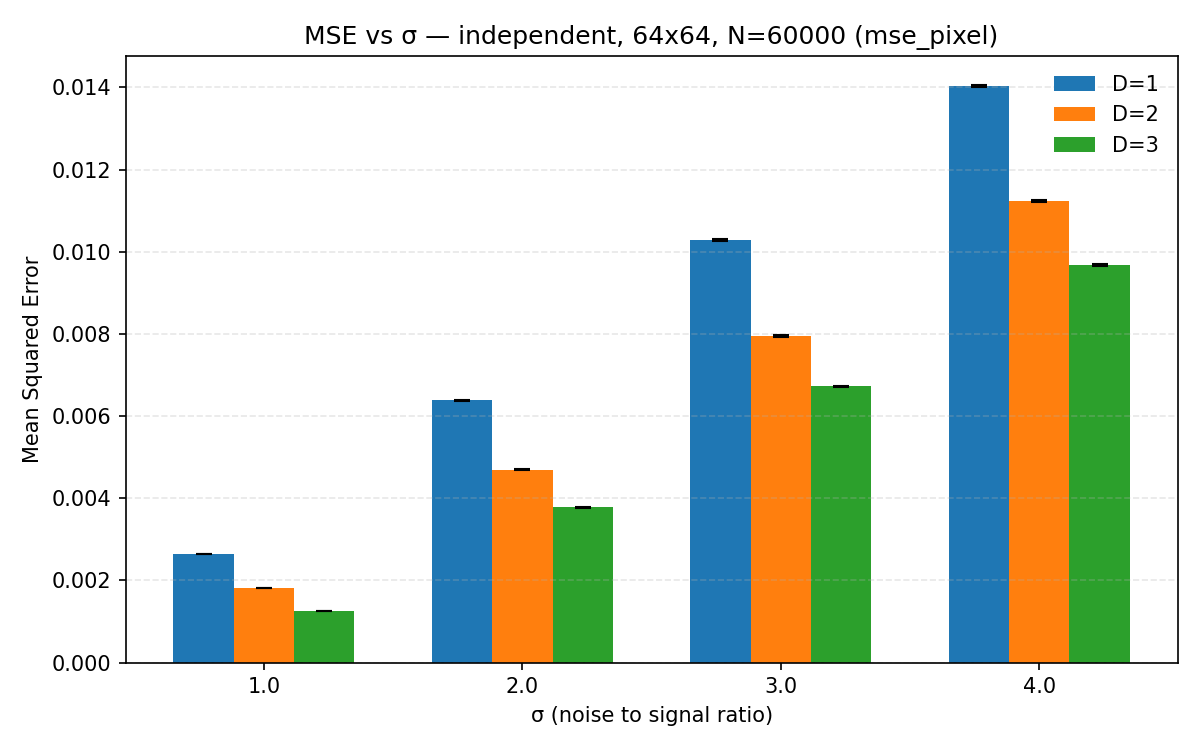}
        \caption{Independent Features}
        \label{fig:indep64}
    \end{subfigure}

    \begin{subfigure}{0.8\linewidth}
        \centering
        \includegraphics[width= 0.9\linewidth]{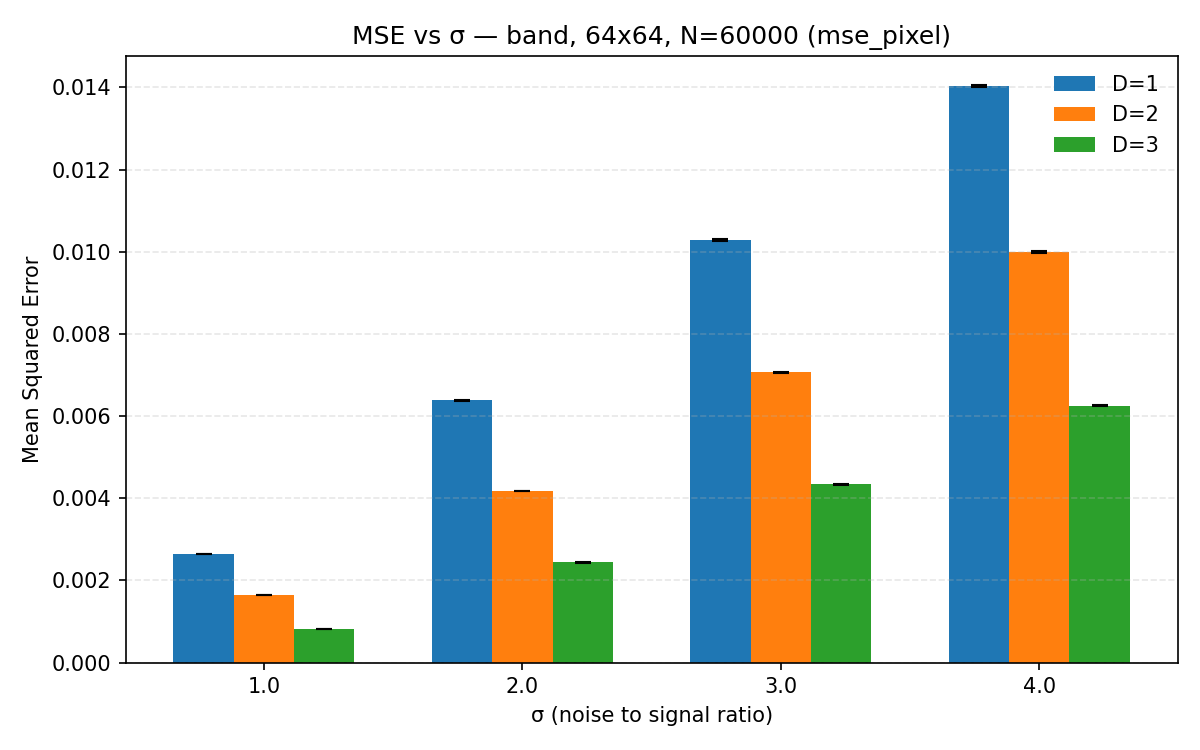}
        \caption{Band-Tied Coupling}
        \label{fig:band64}
    \end{subfigure}
    \caption{\textbf{MNIST-64.} Denoising MSE with \textbf{(a)} independent features, and \textbf{(b)} band-tied coupling on larger images. We see higher degree polynomials are more performant with larger images, compared with Fig. \ref{fig:32}.}
    \vspace{-5mm}
    \label{fig:64}
\end{figure}

\paragraph{Local Correlation}
Across image sizes (see Figure~\ref{fig:32} and Appendix~\ref{app:results}), we observe that introducing local coupling between neighboring coefficients yields a consistent and significant improvement in denoising MSE. This provides empirical support for prior mechanistic explanations of diffusion models \citep{niedoba2025mechanisticexplanationdiffusionmodel,kamb2024analytict}, which emphasize that local interactions are among the most predictive features available to a denoiser. Intuitively, wavelet coefficients corresponding to nearby locations tend to co-vary strongly, especially along edges, textures, and digit strokes \citep{NIPS1999_6a5dfac4, Simoncelli2001NaturalIS}. By capturing this short-range dependency, the model can recover local structure more effectively than when treating coefficients as independent. Interestingly, local correlation maintains benefits across all $\sigma$ regimes, suggesting that spatial neighborhoods retain informative signal even when individual coefficients are noisy.

\paragraph{Changing Image Size} Unlike a global linear score approximation, the wavelet parameterization confines nonlinear interactions to small neighborhoods and scales, so the feature count grows linearly with the number of coefficients rather than combinatorially. The pooled ridge systems of Section 3 are dense in the feature dimension and grow as such, which sets the practical size limit for this implementation. Empirically, the relative performance of our models improves at higher resolutions. Moving from $32\times 32$ to $64\times 64$ yields larger MSE reductions for the same polynomial degree, suggesting that additional scales provide genuinely useful predictive signal rather than simply inflating dimensionality (see Appendix~\ref{app:results}). 

\paragraph{Comparison to Trained Models}
As introduced by \citet{falck2023multiresolutionframeworkunetsapplications}, U-Nets can be viewed as implicitly learning a multiresolution (wavelet-like) representation. In our comparison to trained CNN/U-Net denoisers (Fig.~\ref{fig:models-32}), the learned models outperform our analytic wavelet denoisers across noise levels, as expected. However, two patterns are noteworthy. First, with local coupling and higher polynomial degree, the gap at low to moderate noise ($\sigma\in \{1,2\}$) narrows noticeably, indicating that much of the benefit of learned models may arise from local interactions. Second, as noise increases ($\sigma\ge 3$), the trained models retain a larger advantage, consistent with their capacity to leverage deeper nonlinearity, cross-scale mixing, and data-dependent priors learned during training. Crucially, our approach achieves these results with no gradient-based training and a single closed-form ridge solve per setting, yielding interpretable coefficients and stable behavior under changes in image size.

\section{FUTURE WORK AND CONCLUSION}
\subsection{Future Work}
While this paper is only a first step toward interpreting diffusion models through a wavelet decomposition of the score, it suggests a flexible framework for doing so. Our analytic construction makes it possible to isolate how specific moments, local dependencies, and architectural biases contribute to denoising behavior, and in turn points to several natural directions for future work. On the theoretical side, an important next step is to develop approximation guarantees and finite-sample error bounds for the proposed score estimators. On the modeling side, our experiments focus on a limited family of dependencies, leaving richer forms of multi-scale structure, cross-scale interactions, and other architectural biases to be explored. Experimentally, there are also many settings left to investigate. In particular, although the framework extends naturally to RGB and other multi-channel images, we do not include color-image experiments here. Moreover, we only test at modest image resolutions, whereas prior work such as \citep{Zhang_2025_CVPR} suggests that the advantages of wavelet representations may become more pronounced at very high resolutions. Finally, existing work such as \citep{kamb2024analytict} derives closed-form ideal score machines that allow direct, seed-by-seed comparison between a theoretical denoiser and the output of a learned diffusion model. Developing an analogous ideal score machine in the wavelet basis could provide further insight into the behavior of diffusion models. Finally, we suggest that expanding to video data would be a natural and interesting approach, especially given the natural role of wavelets in temporal modeling. 

\subsection{Conclusion}
We introduced an analytically tractable, wavelet–based parameterization of diffusion scores with closed-form normal equations and three structured dependency families (independent, band-tied, local). This framework isolates which attributes of the data distribution contribute to denoising across noise scales and resolutions. Empirically, higher polynomial degree improves accuracy, band-tying sharpens edges, and local coupling yields the most reliable MSE gains across $\sigma$ and image sizes. The gap to trained CNN/U-Net denoisers narrows at low–moderate noise when locality is encoded explicitly, suggesting that much of their advantage stems from short-range, local interactions that our analytic model captures without gradient-based training. The method is scalable (closed-form per setting), interpretable (moment-level diagnostics), and robust to increasing resolution. Because the score coefficients are solved from moments, the framework serves as a diagnostic tool: it quantifies the value of specific correlations (orientation co-activation vs.\ spatial neighborhoods) and can inform data-efficient architectural choices and initialization schemes for learned models.

\section{ACKNOWLEDGMENTS}
This work was supported by the Kempner Institute for the Study of Natural and Artificial Intelligence at Harvard University. We thank the Kempner for access to compute resources. We are also grateful to the CRISP group at Harvard SEAS for many thoughtful conversations. EF thanks the Calvin Coolidge Presidential Foundation for support during her undergraduate studies. Additionally, we thank Tom de Oliveira for a careful reading that identified an implementation error in an earlier version of the experiments.

\bibliography{ref}

\clearpage
\onecolumn
\section*{Checklist}
\begin{enumerate}
  \item For all models and algorithms presented, check if you include:
  \begin{enumerate}
    \item A clear description of the mathematical setting, assumptions, algorithm, and/or model. Yes, see Section 2.
    \item An analysis of the properties and complexity (time, space, sample size) of any algorithm. Yes, though since the point of the algorithm is interpretability, not efficiency, we only briefly discuss this in Appendix E. 
    \item (Optional) Anonymized source code, with specification of all dependencies, including external libraries. No
  \end{enumerate}

  \item For any theoretical claim, check if you include:
  \begin{enumerate}
    \item Statements of the full set of assumptions of all theoretical results. Yes, full statements and proofs are provided.
    \item Complete proofs of all theoretical results. Yes, some proofs are delayed until Appendix C.
    \item Clear explanations of any assumptions. Yes, I hope. 
  \end{enumerate}

  \item For all figures and tables that present empirical results, check if you include:
  \begin{enumerate}
    \item The code, data, and instructions needed to reproduce the main experimental results (either in the supplemental material or as a URL). Mostly yes, see Appendix B and the discussion of the implementation.
    \item All the training details (e.g., data splits, hyperparameters, how they were chosen). Yes, see Appendix B.
    \item A clear definition of the specific measure or statistics and error bars (e.g., with respect to the random seed after running experiments multiple times). Yes, see Section 4. 
    \item A description of the computing infrastructure used. (e.g., type of GPUs, internal cluster, or cloud provider). Yes, see Appendix B.
  \end{enumerate}

  \item If you are using existing assets (e.g., code, data, models) or curating/releasing new assets, check if you include:
  \begin{enumerate}
    \item Citations of the creator If your work uses existing assets. Yes, but we only use MNIST, a very standard dataset. 
    \item The license information of the assets, if applicable. Not applicable. 
    \item New assets either in the supplemental material or as a URL, if applicable. Not Applicable.
    \item Information about consent from data providers/curators. Not Applicable
    \item Discussion of sensible content if applicable, e.g., personally identifiable information or offensive content. Not Applicable.
  \end{enumerate}
  \item If you used crowdsourcing or conducted research with human subjects, check if you include:
  \begin{enumerate}
    \item The full text of instructions given to participants and screenshots. Not Applicable
    \item Descriptions of potential participant risks, with links to Institutional Review Board (IRB) approvals if applicable. Not Applicable.
    \item The estimated hourly wage paid to participants and the total amount spent on participant compensation. Not Applicable.
  \end{enumerate}
\end{enumerate}
\pagebreak
\newpage
\appendix
\onecolumn
\section{Additional Results}
\label{app:results}

\subsection{Comparison with Linear Baseline}
In Figure \ref{fig:linear_baseline} below, we compare our proposed framework with the simple linear baseline of \cite{wang2024unreasonableeffectivenessgaussianscore} (detailed in Section \ref{app:lin_denoiser_baseline}).  From this comparison, we see very clearly that the wavelet decomposition of the score approximator is most beneficial for higher resolution images (64x64, right) where our model significantly outperforms the linear counterpart. Conversely, on smaller images (32x32, left), the frequency-focused wavelet decomposition struggles to deal with the high frequency components induced by low spatial resolution. We note that, somewhat analogously, in prior work, it has been found that wavelet decompositions are helpful for diffusion models operating on high-resolution images, since wavelet coefficients compactly represent images as multi-channel, low-resolution feature maps \citep{hoogeboom2023simplediffusionendtoenddiffusion}. In future work, we intend to explore the applicability of our model to higher resolution images more thoroughly, as we think this may be a key area where this method has significant advantages over existing solutions.

\begin{figure}[h]
    \centering
    \begin{subfigure}{0.48\textwidth}
        \centering
        \includegraphics[width=\linewidth]{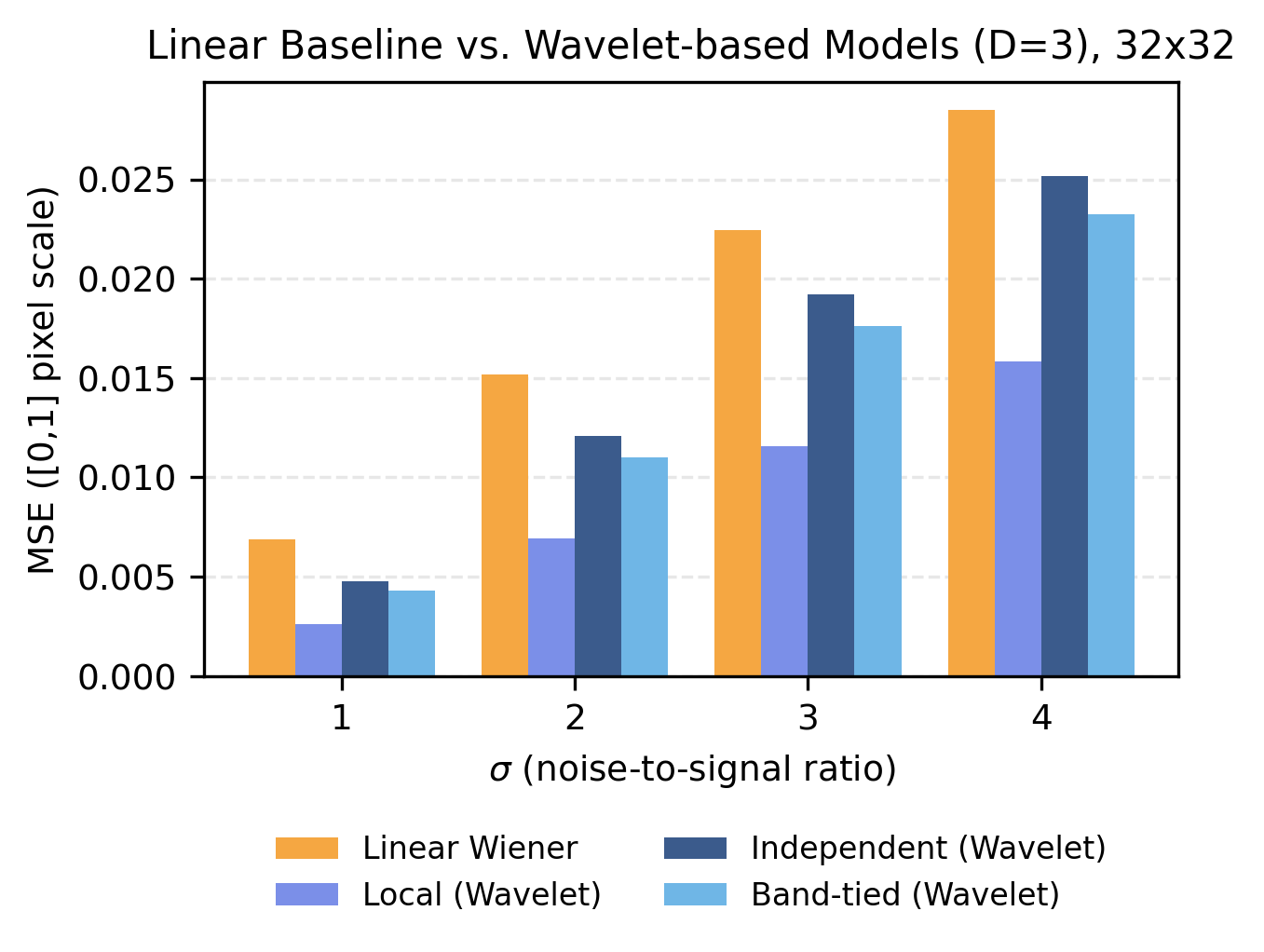}
        \caption{MNIST-32: Linear baseline vs. wavelet models (D=3).}
        \label{fig:linear_baseline_32}
    \end{subfigure}\hfill
    \begin{subfigure}{0.48\textwidth}
        \centering
        \includegraphics[width=\linewidth]{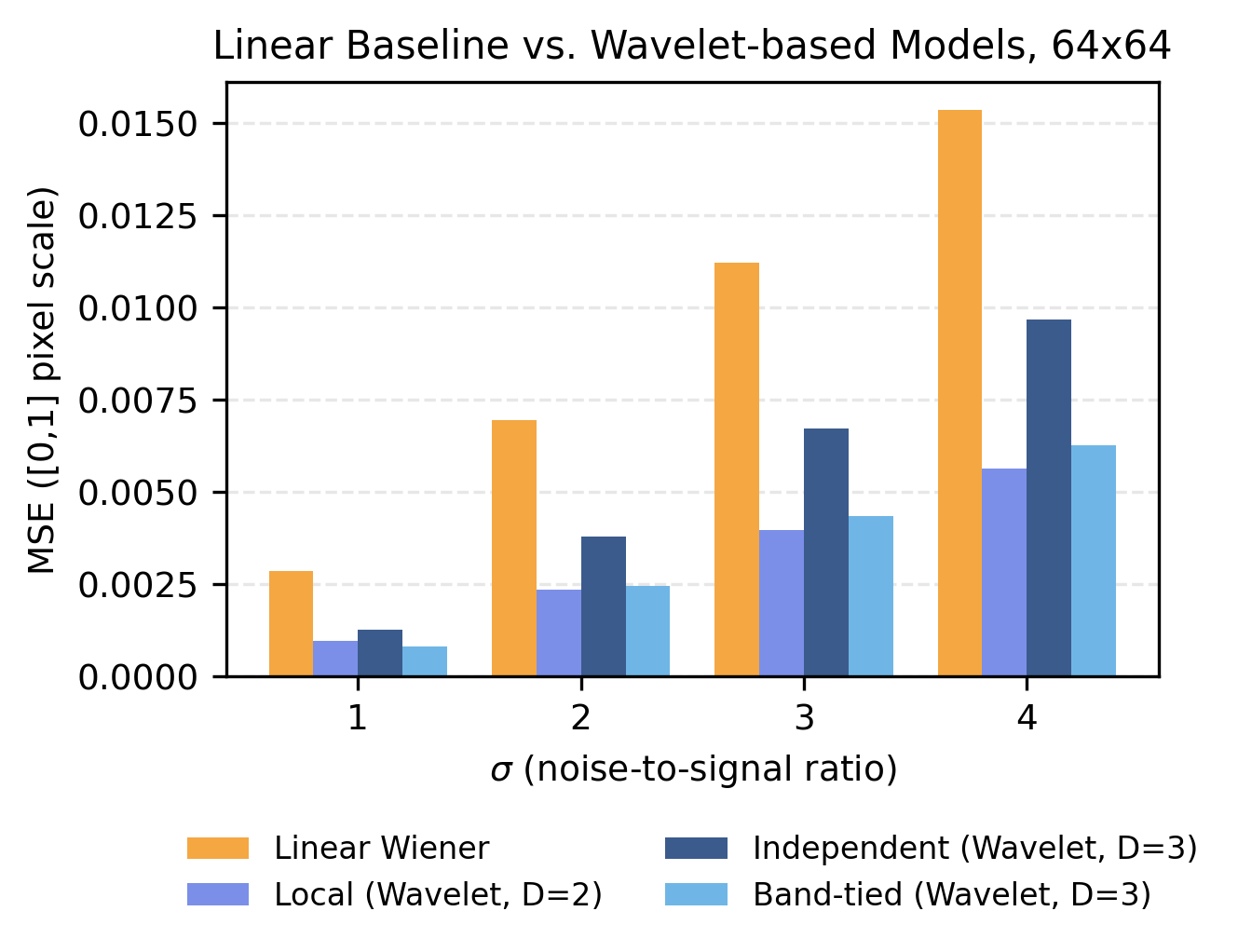}
        \caption{MNIST-64: Linear baseline vs. wavelet models.}
        \label{fig:linear_baseline_64}
    \end{subfigure}
    \caption{Comparison of linear baseline vs. wavelet models (D=3) on MNIST-32 and MNIST-64. We see the wavelet decomposition of the score significantly outperforms the linear baseline.}
    \label{fig:linear_baseline}
    \vspace{-5mm}
\end{figure}

\subsection{Hermite Polynomial Results}
In Figure \ref{fig:hermite} below we also include results from an independent model expanded in terms of the probabilist's Hermite polynomials, as described in Section 3. Results are plotted for both 32 by 32 images and 64 by 64 images. The setup is identical to that described in Sections 3 and 4. 
\begin{figure}[h]
    \centering
    \begin{subfigure}{0.48\textwidth}
        \centering
        \includegraphics[width=\linewidth]{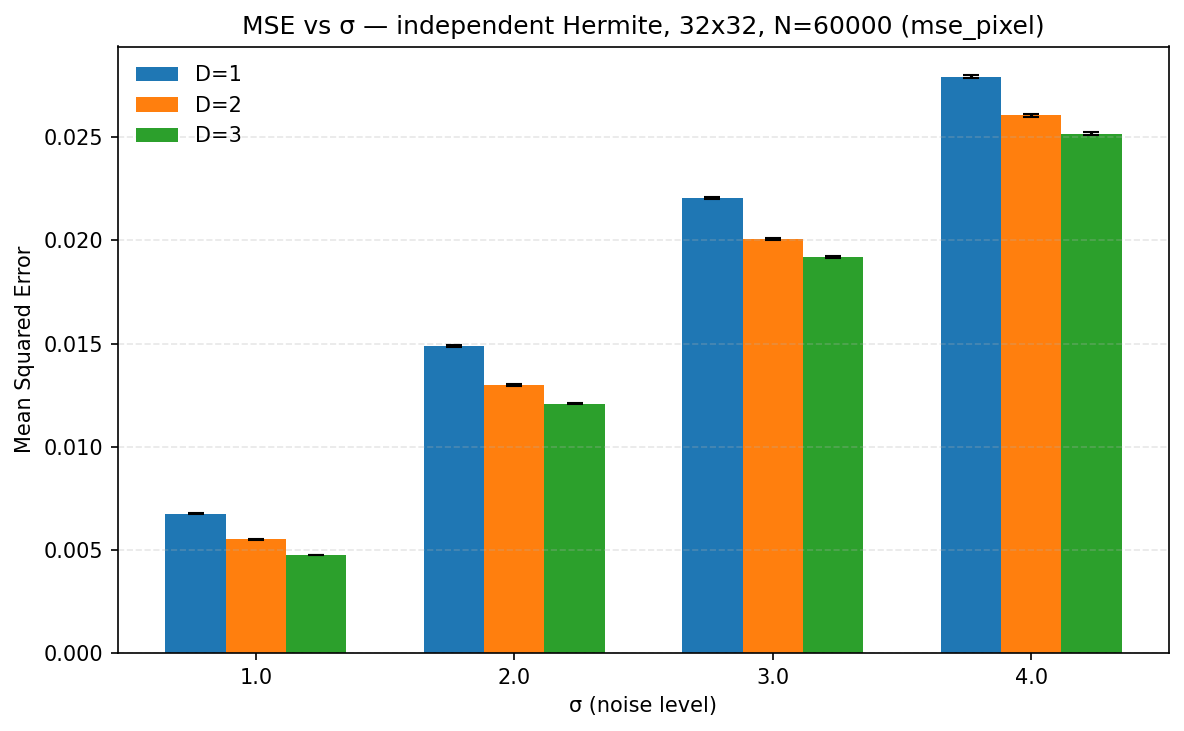}
        \caption{MNIST-32: Hermite Polynomials of degrees 1,2, and 3. }
        \label{fig:hermite_32}
    \end{subfigure}\hfill
    \begin{subfigure}{0.48\textwidth}
        \centering
    \includegraphics[width=\linewidth]{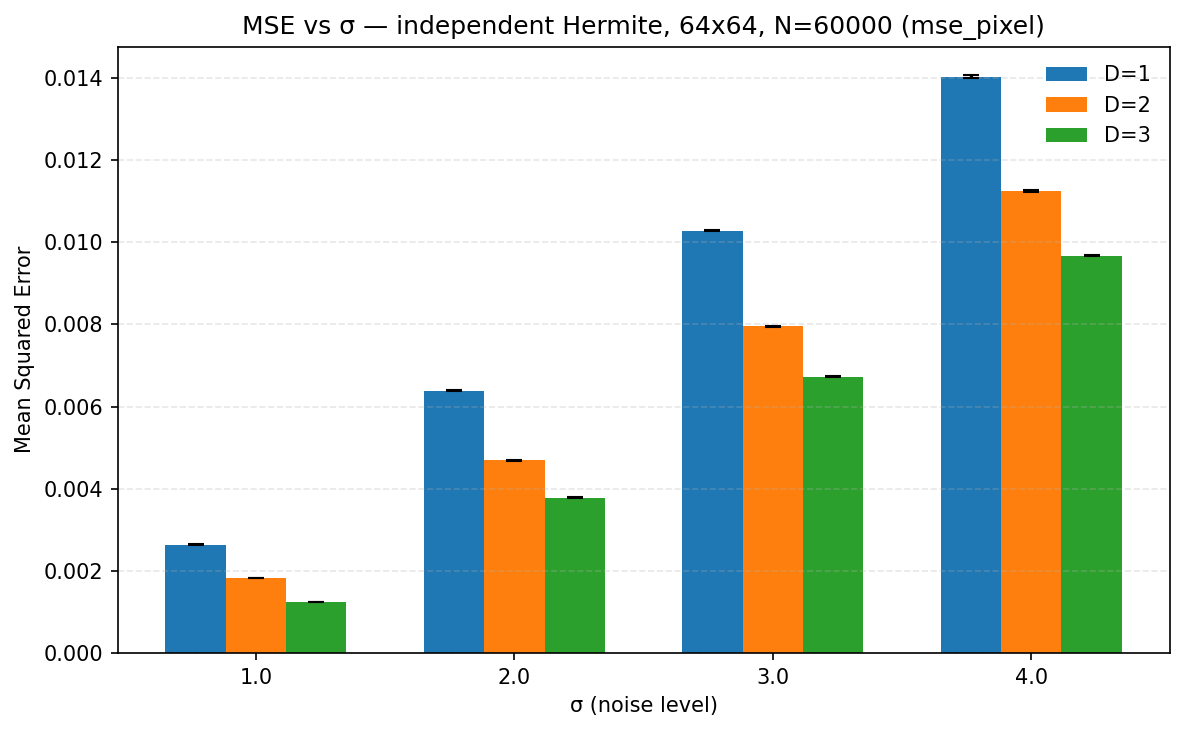}
        \caption{MNIST-64: Hermite Polynomials of degrees 1, 2, and 3.}
        \label{fig:hermite_64}
    \end{subfigure}
    \caption{Comparison of independent Hermite polynomials across different image sizes and degrees. Increasing the degree of the approximation is most helpful for larger images.}
    \label{fig:hermite}
    \vspace{-5mm}
\end{figure}

\subsection{Qualitative Denoising Results}

We also display qualitative denoising results below in Figure \ref{fig:gen_indep} for the independent model, and Figure\ref{fig:gen_band} for the band-tied model. We see that our model is able to denoise qualitatively well across a range of settings, and matches the quantitative results implying better performance for higher resolution images. 

\begin{figure}[htbp]
    \centering
    \includegraphics[width=0.7\textwidth]{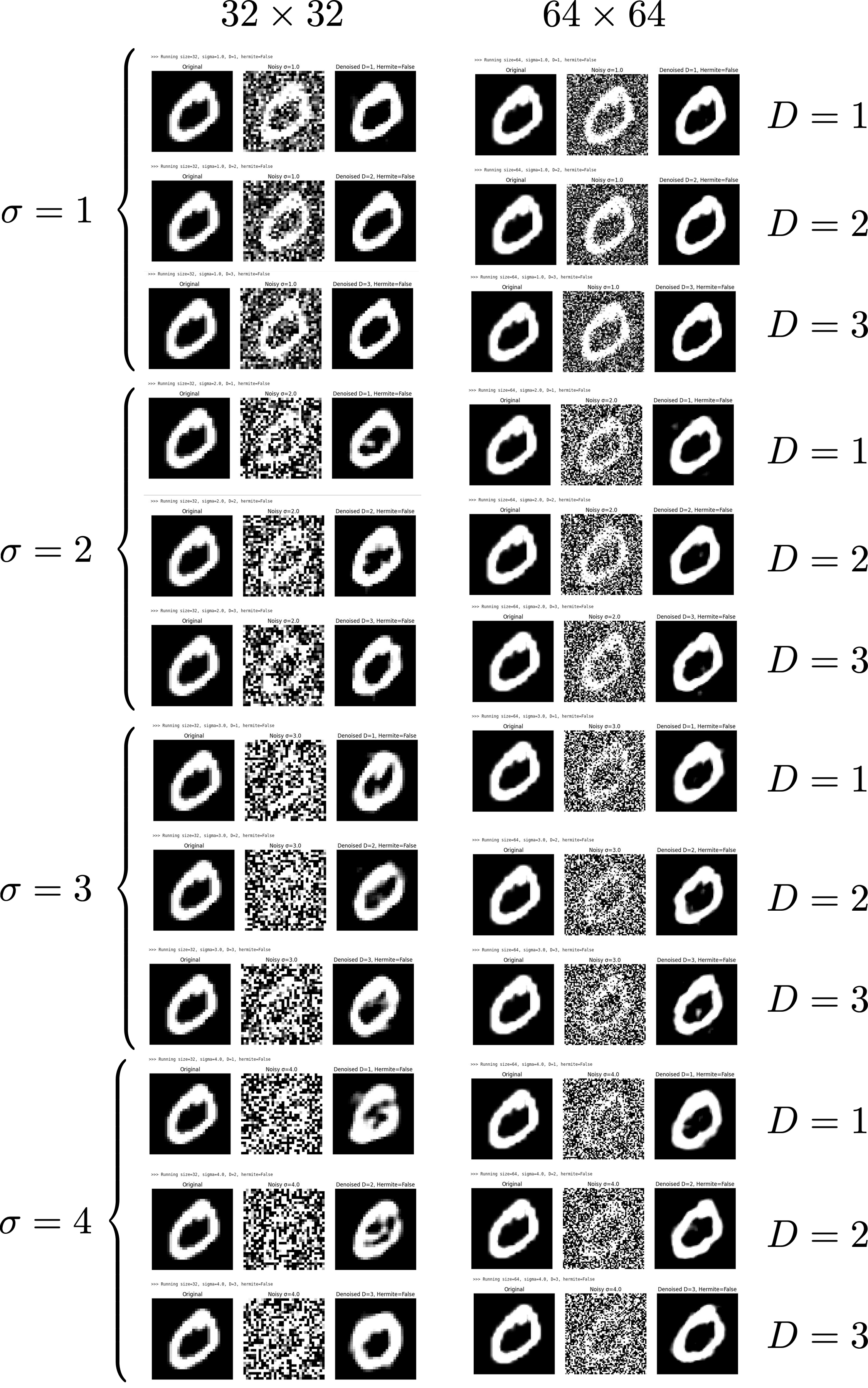}
    \caption{Denoising results for the independent monomial model on two image sizes (columns) and across ranges of noise levels (groups of 3 rows), for different degrees D.}
    \label{fig:gen_indep}
\end{figure}

\begin{figure}[htbp]
    \centering
    \includegraphics[width=0.4\textwidth]{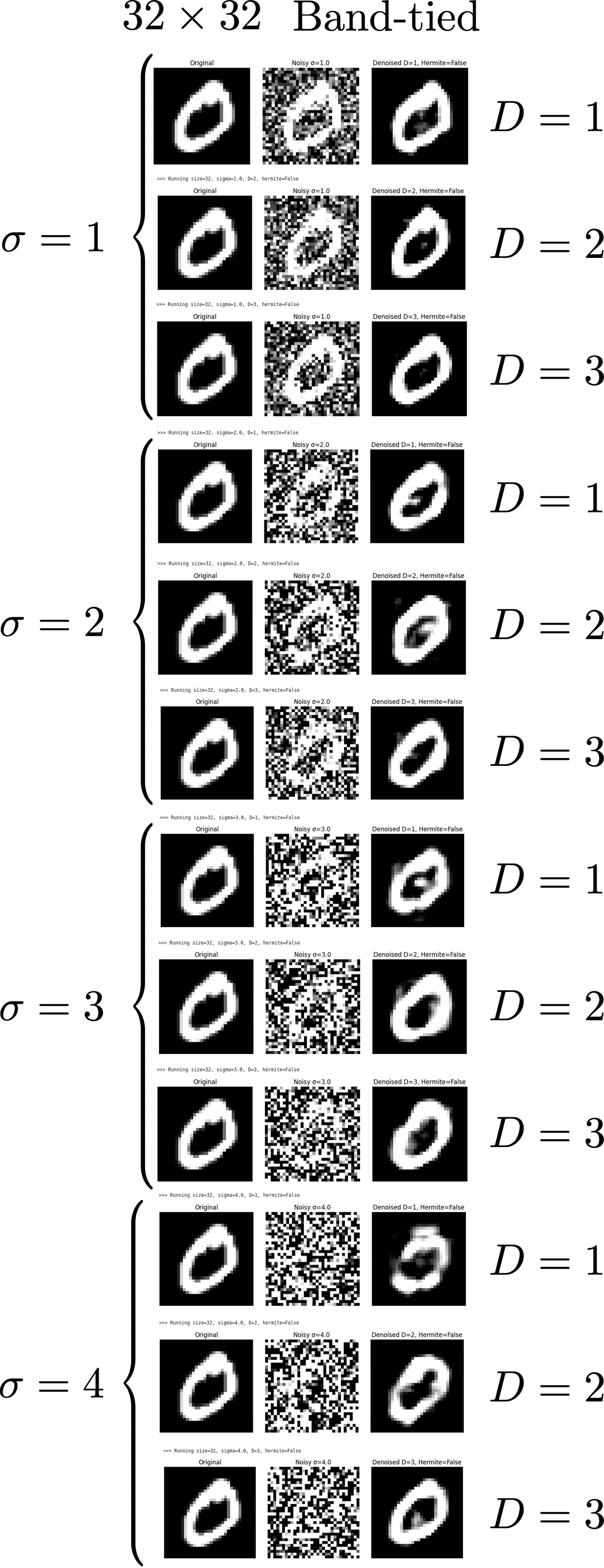}
    \caption{Denoising results for the band-tied model on 32 by 32 images, across ranges of noise levels (groups of 3 rows), for different degrees D.}
    \label{fig:gen_band}
\end{figure}

\subsubsection{Diffusion Model Generations}
\begin{figure}[htbp]
    \centering
    \begin{subfigure}[b]{0.4\textwidth}
        \centering
        \includegraphics[width=\textwidth]{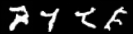}
        \caption{EDM, 1 block, no attention}
    \end{subfigure}
    
    \vspace{0.3em}
    \begin{subfigure}[b]{0.4\textwidth}
        \centering
        \includegraphics[width=\textwidth]{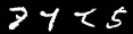}
        \caption{EDM, 2 blocks, no attention}
    \end{subfigure}
    
    \vspace{0.3em}
    \begin{subfigure}[b]{0.4\textwidth}
        \centering
        \includegraphics[width=\textwidth]{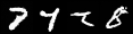}
        \caption{EDM, 4 blocks, no attention}
    \end{subfigure}
    
    \vspace{0.3em}
    \begin{subfigure}[b]{0.4\textwidth}
        \centering
        \includegraphics[width=\textwidth]{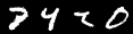}
        \caption{EDM, 4 blocks, denser}
    \end{subfigure}
    
    \vspace{0.3em}
    \begin{subfigure}[b]{0.4\textwidth}
        \centering
        \includegraphics[width=\textwidth]{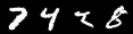}
        \caption{EDM, 4 blocks, small LR, long training}
    \end{subfigure}
    
    \vspace{0.3em}
    \begin{subfigure}[b]{0.4\textwidth}
        \centering
        \includegraphics[width=\textwidth]{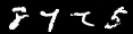}
        \caption{EDM, 4 blocks, width 64}
    \end{subfigure}
    
    \vspace{0.3em}
    \begin{subfigure}[b]{0.4\textwidth}
        \centering
        \includegraphics[width=\textwidth]{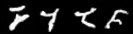}
        \caption{EDM, shallow 1 block}
    \end{subfigure}
    
    \vspace{0.3em}
    \begin{subfigure}[b]{0.4\textwidth}
        \centering
        \includegraphics[width=\textwidth]{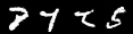}
        \caption{EDM, shallow}
    \end{subfigure}
    
    \vspace{0.3em}
    \begin{subfigure}[b]{0.4\textwidth}
        \centering
        \includegraphics[width=\textwidth]{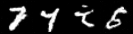}
        \caption{EDM, deeper 1 block}
    \end{subfigure}
    \caption{MNIST samples generated by different UNet-CNN EDM diffusion model variants. Each row shows images from one model configuration.}
    \label{fig:mnist_edm_models_vertical}
\end{figure}

\begin{figure}[htbp]
    \centering
    \begin{subfigure}[b]{0.4\textwidth}
        \centering
        \includegraphics[width=\textwidth]{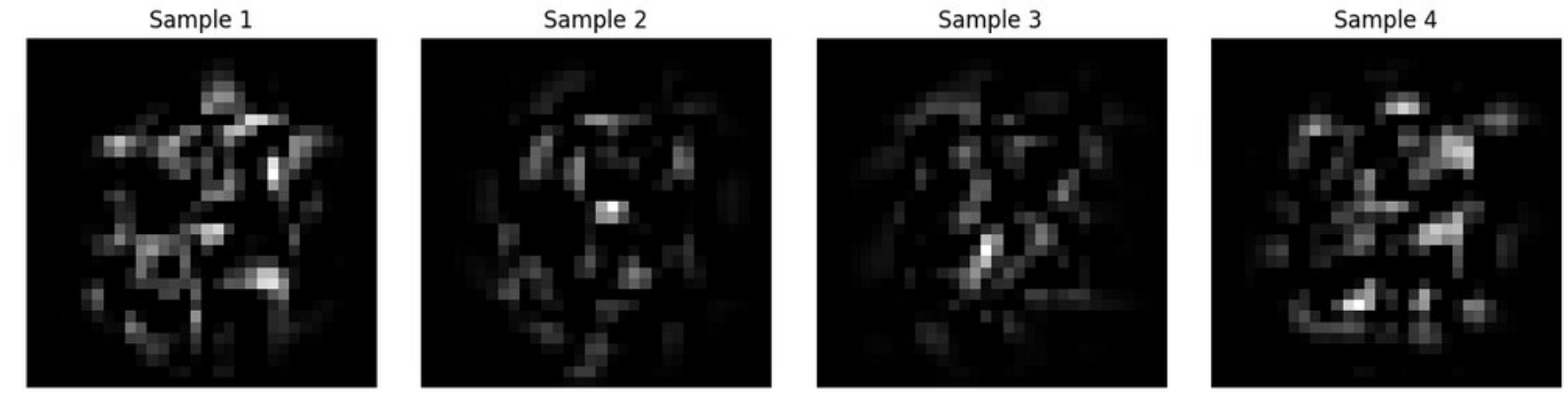}
        \caption{Independent, degree 1}
    \end{subfigure}
    \hfill
    \begin{subfigure}[b]{0.4\textwidth}
        \centering
        \includegraphics[width=\textwidth]{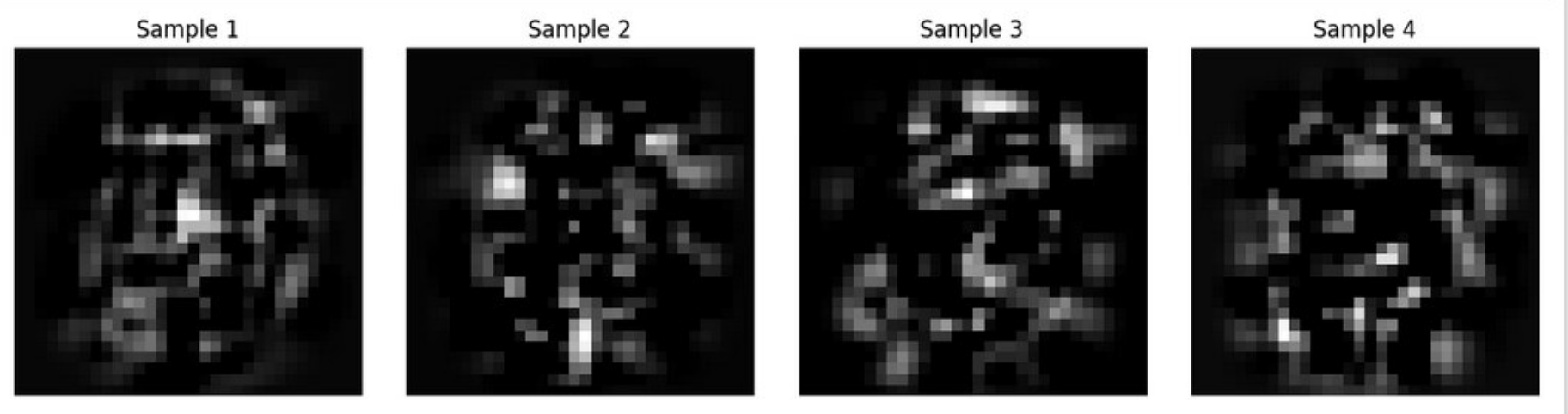}
        \caption{Band, degree 1}
    \end{subfigure}
    
    \vspace{0.5em}
    
    \begin{subfigure}[b]{0.4\textwidth}
        \centering
        \includegraphics[width=\textwidth]{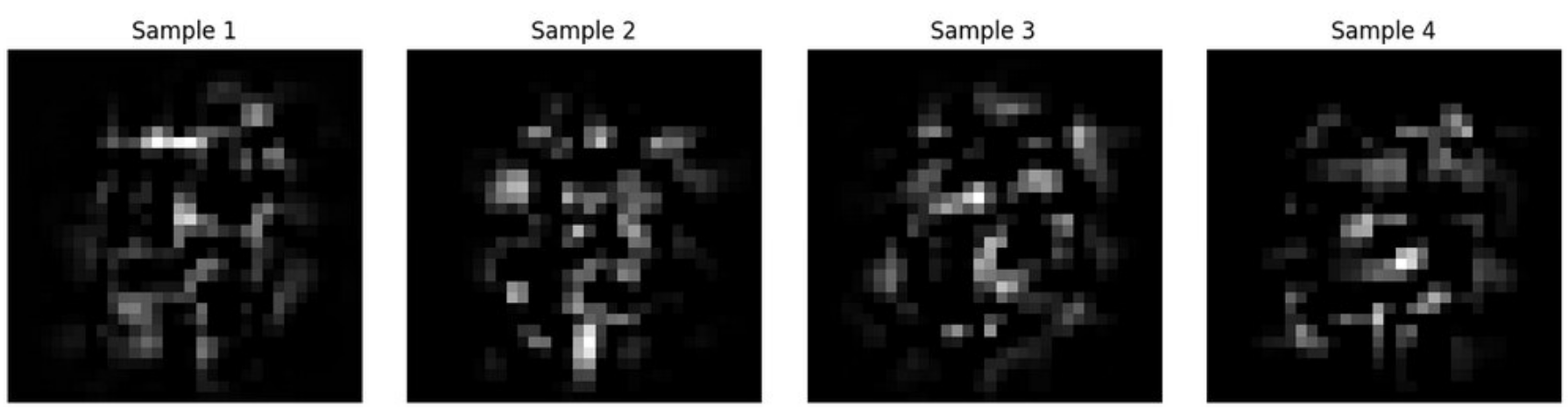}
        \caption{Independent, degree 2}
    \end{subfigure}
    \hfill
    \begin{subfigure}[b]{0.4\textwidth}
        \centering
        \includegraphics[width=\textwidth]{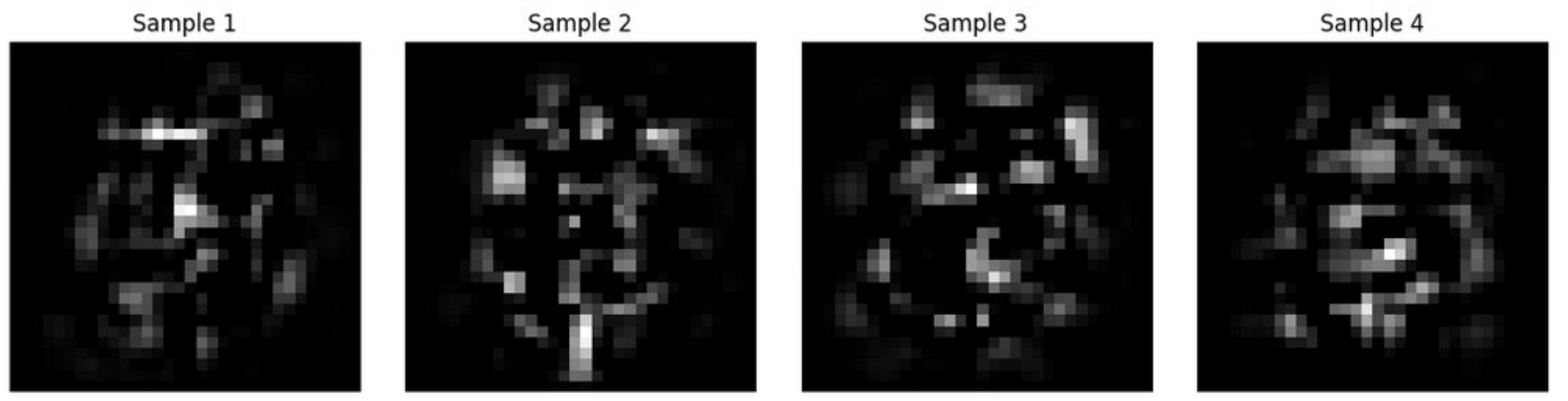}
        \caption{Band, degree 2}
    \end{subfigure}
    
    \vspace{0.5em}
    
    \begin{subfigure}[b]{0.4\textwidth}
        \centering
        \includegraphics[width=\textwidth]{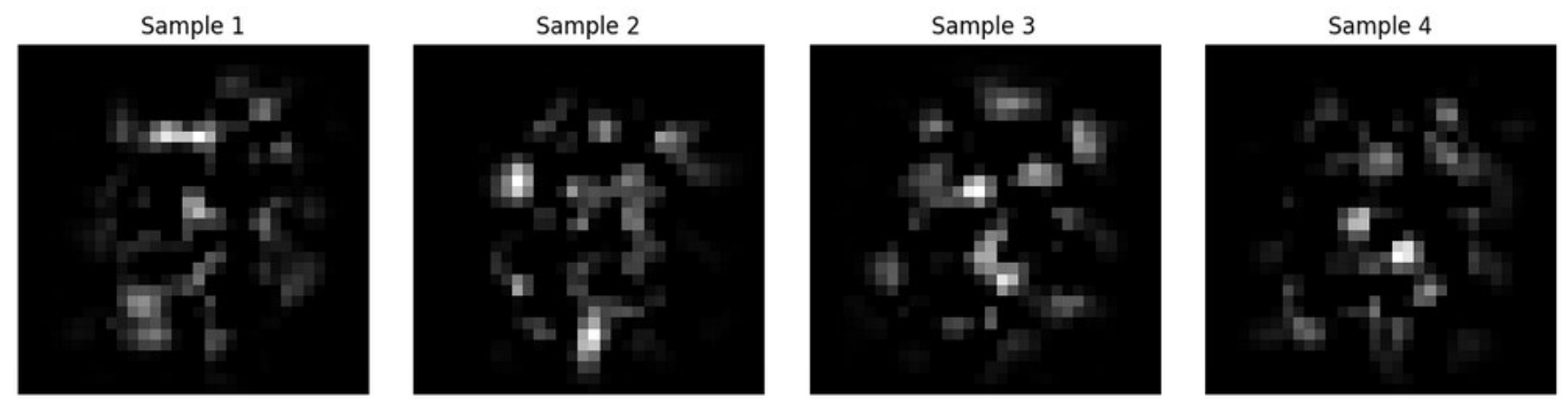}
        \caption{Independent, degree 3}
    \end{subfigure}
    \hfill
    \begin{subfigure}[b]{0.4\textwidth}
        \centering
        \includegraphics[width=\textwidth]{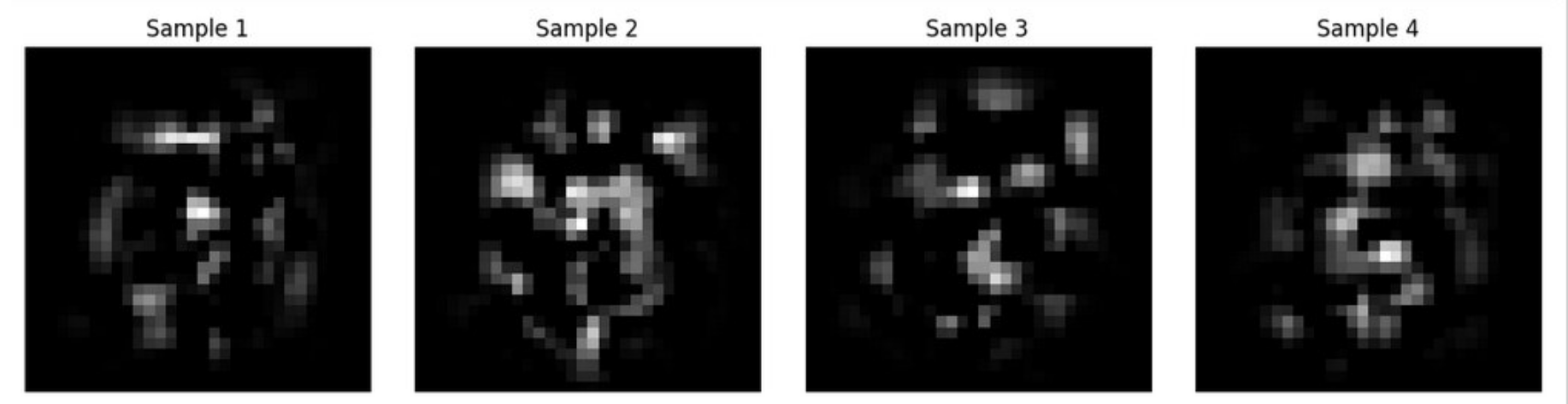}
        \caption{Band, degree 3}
    \end{subfigure}
    
    \caption{Generated samples for Independent vs. Band models at varying polynomial degrees. Each column shows one model family; each row corresponds to a degree.}
    \label{fig:indep_band_degree_comparison}
\end{figure}

While the overall quality of the generated images is low, the broad shape of features is promising, especially in the band tied high degree cases. However, the wavelet-based diffusion process appears to capture some coarse structural patterns. In particular, the band-tied, higher-degree models show a more coherent emergence of organized intensity regions, suggesting that the model begins to internalize meaningful global dependencies. These results indicate that, despite the pixel-level noise, the generative process is learning the broad spatial organization of features, which could serve as a foundation for higher-fidelity synthesis with improved architectures or longer training.

\section{Training Details}
\label{app:training}
\subsection{UNet denoiser baseline details}

\paragraph{Preconditioning and objective}
We followed the same preconditioning and loss weighting scheme from \cite{karras2022edm}, with the code below. 

\begin{RoundedListing}
class EDMPrecondWrapper(nn.Module):
    def __init__(self, model, sigma_data=0.5, sigma_min=0.002, sigma_max=80, rho=7.0):
        super().__init__()
        self.model = model
        self.sigma_data = sigma_data
        self.sigma_min = sigma_min
        self.sigma_max = sigma_max
        self.rho = rho
        
    def forward(self, X, sigma, cond=None, ):
        sigma[sigma == 0] = self.sigma_min
        ## edm preconditioning for input and output
        ## https://github.com/NVlabs/edm/blob/main/training/networks.py#L632
        # unsqueze sigma to have same dimension as X (which may have 2-4 dim) 
        sigma_vec = sigma.view([-1, ] + [1, ] * (X.ndim - 1))
        c_skip = self.sigma_data ** 2 / (sigma_vec ** 2 + self.sigma_data ** 2)
        c_out = sigma_vec * self.sigma_data / (sigma_vec ** 2 + self.sigma_data ** 2).sqrt()
        c_in = 1 / (self.sigma_data ** 2 + sigma_vec ** 2).sqrt()
        c_noise = sigma.log() / 4
        model_out = self.model(c_in * X, c_noise, cond=cond)
        return c_skip * X + c_out * model_out
\end{RoundedListing}

\begin{RoundedListing}
class EDMLoss:
    def __init__(self, P_mean=-1.2, P_std=1.2, sigma_data=0.5):
        self.P_mean = P_mean
        self.P_std = P_std
        self.sigma_data = sigma_data

    def __call__(self, net, X, labels=None, ):
        rnd_normal = torch.randn([X.shape[0],] + [1, ] * (X.ndim - 1), device=X.device)
        # unsqueeze to match the ndim of X
        sigma = (rnd_normal * self.P_std + self.P_mean).exp()
        weight = (sigma ** 2 + self.sigma_data ** 2) / (sigma * self.sigma_data) ** 2
        # maybe augment
        n = torch.randn_like(X) * sigma
        D_yn = net(X + n, sigma, cond=labels, )
        loss = weight * ((D_yn - X) ** 2)
        return loss
\end{RoundedListing}

\paragraph{Model variants and training hyperparameters}
We adopted the U-Net architecture from \cite{song2019generative} as the backbone and systematically varied its architectural hyperparameters. Specifically, we trained U-Nets using the EDM objective on MNIST, varying the number of \texttt{blocks} and the base width (\texttt{model\_channels}). Each encoder block halves the spatial resolution and multiplies the number of channels by \texttt{channel\_mult}, while each decoder block performs the inverse upscaling by $\times 2$. We fixed \texttt{layers\_per\_block=1}, so each block contains one residual layer. The \texttt{channel\_mult} tuple defines the width scaling across resolution levels—for example, with \texttt{base\_channels=16} and \texttt{channel\_mult=(1, 2, 3, 4)}, the block widths are 16, 32, 48, and 64. Unless otherwise noted, all models use \texttt{batch\_size=2048}, \texttt{attn\_resolutions=[]}, and \texttt{decoder\_init\_attn=False}, i.e., no attention is applied in the U-Net.

For training, we used Adam optimizer and constant learning rate without scheduling. 


\begin{table}[!ht]
\centering
\caption{\textbf{UNet-EDM architectures and training schedules.} \\
Shared hyperparameters: \texttt{layers\_per\_block=1}, \texttt{decoder\_init\_attn=False}, \texttt{attn\_resolutions=[]} 
}
\label{tab:mnist-edm-variants}
\setlength{\tabcolsep}{4pt}
\begin{tabular}{lccccccr}
\toprule
\textbf{Experiment} & \textbf{Blocks} & \textbf{Base\_ch} & \textbf{ch\_mult} & \textbf{Steps} & \textbf{Batch} & \textbf{LR} & \textbf{Train time} \\
\midrule
\texttt{4blocks\_noattn} & 4 & 16 & 1,2,3,4 & 10k & 2048 & 1e-3 & $\sim$25 min \\
\texttt{4blocks\_noattn\_smalllr\_longtrain} & 4 & 16 & 1,2,3,4 & 50k & 2048 & 1e-4 & $\sim$2h 4min \\
\texttt{4blocks\_noattn\_denser} & 4 & 16 & 1,2,3,4 & 10k & 2048 & 1e-3 & $\sim$37 min \\
\texttt{4blocks\_wide64\_noattn} & 4 & 64 & 1,2,3,4 & 10k & 2048 & 1e-3 & $\sim$77 min \\
\texttt{1block\_noattn} & 1 & 16 & 1 & 10k & 2048 & 1e-3 & $\sim$28 min \\
\texttt{1block\_wide128\_noattn}$^*$ & 1 & 32 & 1 & 10k & 2048 & 1e-3 & $\sim$17 min \\
\texttt{2blocks\_noattn} & 2 & 16 & 1,2 & 10k & 2048 & 1e-3 & $\sim$22 min \\
\bottomrule
\end{tabular}
\end{table}

\subsection{Linear denoiser baseline details}\label{app:lin_denoiser_baseline}
As another baseline, we computed the optimal linear denoiser, which captures the Gaussian statistics of the dataset. 
We computed the empirical mean and covariance of the dataset $\mu,\Sigma$.  
Then we used the following formula for linear denoiser \citep{wang2024unreasonableeffectivenessgaussianscore}, 
\begin{align}
    \mathbf{D}_{lin}(\mathbf{x},\sigma)&=\mu+\Sigma(\Sigma+\sigma^2I)^{-1}(\mathbf{x}-\mu)\\
    &=\mu+U\text{diag}\left(\frac{\lambda_k}{\lambda_k+\sigma^2}\right)U^\top(\mathbf{x}-\mu)
\end{align}
where $\Sigma=U\Lambda U^\top$ and $\Lambda=\text{diag}(\lambda_1,\lambda_2,...)$ is the eigen decomposition of the data covariance. In practice, we precomputed the eigendecomposition to avoid explicit matrix inverse.

\subsection{Computing infrastructure}\label{app:compute_resources}
For all DNN training runs and numerical simulation of the theory, we used an academic cluster with A100, H100 GPUs. Generally the UNet training runs on MNIST finish from 25mins to 2hrs on one A100 GPU. 

\section{Additional Proofs}

\subsection{Relationship between the Pixel-space Regression and Coefficient-wise Regression}
\label{app:math}
\paragraph{Notation}
Let $x \in \mathbb{R}^d$ be an image, with $d = H \times W$.  
Let $\{ \omega_j \}_{j \in \mathcal{I}} \subset \mathbb{R}^d$ be a wavelet frame with matrix 
$B \in \mathbb{R}^{M \times d}$, where each row is a wavelet $\omega_j$ and we assume that $\text{rank}(B)=d$.
We note that in practice this wavelet basis $B \in \mathbb R^{M \times d}$ is over-complete.  

We corrupt images as
\[
X_t = X_0 + \sigma Z, \quad Z \sim \mathcal{N}(0, I_d).
\]
In the code, we use a shared feature map
\[
\Phi(x) \in \mathbb{R}^p
\]
across all output coordinates. The denoiser network output is a linear readout over the features
\[
\widehat{X}_0(x) = \Phi(x) W, \quad W \in \mathbb{R}^{P \times d}.
\]

\subsubsection*{Lemma 1}
\textbf{Statement.}  [The pixel-level denoiser is the same as the wavelet denoiser]  Let $d = H \times W$ and $B \in \mathbb{R}^{M \times d}$ be a dictionary with unit norm rows $w_i$ consisting of wavelet atoms. Assume that $\text{rank(B) = d}$. Let $F:\mathbb{R}^M\to\mathbb{R}^P$ be any fixed feature map and writing $\Phi(x) := F(xB^\top)$ and denote the stacked features as $\Phi$. Then, we have that the two ridge problems yield the same optima. 
\begingroup\small
\begin{equation}
\nonumber
    \text{\bf(Pixel)}\;\
W_\gamma^\star \in \arg\min_{W\in\mathbb{R}^{P\times d}}
\; \mathbb{E}\bigl[\|\Phi W - X_0\|^2\bigr] + \gamma\|W\|_F^2,
\end{equation}
\endgroup
\begingroup\small
\begin{equation}
    \text{\bf(Coeff)}\;\
U_\gamma^\star \in \arg\min_{U\in\mathbb{R}^{P\times M}}
\; \mathbb{E}\bigl[\|\Phi U - X_0B^\top\|^2\bigr] + \gamma\|U\|_F^2.
\end{equation}
\endgroup
Then the minimizers satisfy $U^\star \;=\; W^\star B^\top$ and for any left inverse $R \in \mathbb{R}^{d \times M}$ of $B$, the denoised predictions are the same. 

\begin{proof}
Since $\gamma > 0$, both objectives are strictly convex and each has a unique minimizer whose form is easy to find from the standard ridge normal equations. Writing $$ G : = (\Phi^T \Phi + \gamma I_p)^{-1} \Phi^{T}$$

\noindent Then, $W^* = G X_0$ and $U^* = G(X_0 B^T)$. Note that $G$ is only a function of $\Phi$ and $\gamma$. Therefore, we have
$$ U^* = G(X_0 B^T) = (GX_0) B^T = W^* B^T$$

\noindent Finally, notice that if $R$ is a left inverse of $B$, meaning $RB = I_d$, we have that for any $x$, 

$$\Phi(x) U^* R^T = \Phi(x)(W^* B^T) R^T = \Phi(x) W^* (R B)^T = \Phi(x) W^*$$

\noindent as claimed.
\end{proof}

\section{Extension to Color Images}
In the main text, we write the theory for single channel (gray scale) images. 
For multi-channel (i.e. color) images, the theory applies analogously. 
Take RGB images as example, the wavelet filters are applied to each channel separately, resulting in three sets of wavelet coefficients (R, G, B). 
Then the nonlinear features can be built on top of the concatenated wavelet coefficients. 

\section{Computational Considerations for Larger Images}
In the main text, we run experiments on images of size $32 \times 32$ and $64 \times 64$, while preliminary experiments at larger resolutions ($128 \times 128$ and $256 \times 256$) led to memory requirements that exceeded what our experimental setup could handle. Allowing arbitrary degree-D multivariate polynomials over all n wavelet coordinates leads to combinatorial growth in the number of interaction terms. We therefore restrict the nonlinear interaction structure to structured, interpretable families (independent, band-tied, and local-coupled) which determine which cross-coefficient monomials are included. In the closed-form derivations above, each coefficient is fit only on its own restricted features; in the regression experiments of Section 4, the resulting features are pooled into a single dictionary with an unrestricted linear readout (see Section 3), so the families differ in which interactions are available rather than in which coefficients each output may depend on linearly. At the same time, we believe the wavelet parameterization is more amenable to higher-resolution images than naive linear estimators, since it preserves locality and multiscale structure. Empirically, Figures \ref{fig:32}, \ref{fig:64}, and \ref{fig:linear_baseline} suggest that the relative advantage of the wavelet-based model may become stronger as image size increases.

\newpage

\end{document}